\documentclass[11pt]{article}
\usepackage[utf8]{inputenc}
\usepackage[T1]{fontenc}
\usepackage{lmodern}
\usepackage{geometry}
\usepackage{amsmath,amssymb}

\usepackage{tikz,pgfplots} 
\usepgfplotslibrary{fillbetween}

\usepackage{amsthm}

\newtheorem{theorem}{Theorem}[section]

\renewcommand{\qedsymbol}{$\blacksquare$}

\makeatletter
\@ifundefined{proof}{%
  \newenvironment{proof}[1][Proof]{\par\noindent\textbf{#1.}\ }{\hfill\qedsymbol\par}
}{}
\makeatother

\usepackage{booktabs}
\usepackage{hyperref}
\usepackage{graphicx}
\usepackage{xcolor}
\usepackage{enumitem}
\hypersetup{
  colorlinks=true,
  linkcolor=blue!50!black,
  urlcolor=blue!50!black,
  citecolor=blue!50!black
}
\title{An Operational Kardashev-Style Scale for Autonomous AI - Towards AGI and Superintelligence}
\author{Przemyslaw Chojecki \\\small ulam.ai}
\date{November 17, 2025}

\begin{document}
\maketitle

\begin{abstract}
We propose a Kardashev-inspired yet operational \textbf{Autonomous AI (AAI) Scale} that measures the progression from fixed robotic process automation (AAI-0) to full artificial general intelligence (AAI-4) and beyond. Unlike narrative ladders, our scale is multi-axis and testable. We define ten capability axes (Autonomy, Generality, Planning, Memory/Persistence, Tool Economy, Self-Revision, Sociality/Coordination, Embodiment, World-Model Fidelity, Economic Throughput) aggregated by a composite \textbf{AAI-Index} (a weighted geometric mean). We introduce a measurable \textbf{Self-Improvement Coefficient} $\kappa$ (capability growth per unit of agent-initiated resources) and two \textbf{closure} properties (maintenance and expansion) that convert ``self-improving AI'' into falsifiable criteria. We specify \textbf{OWA-Bench}, an open-world agency benchmark suite that evaluates long-horizon, tool-using, persistent agents. We define \textbf{level gates} for AAI-0\ldots AAI-4 using thresholds on the axes, $\kappa$, and closure proofs. Synthetic experiments illustrate how present-day systems map onto the scale and how the \textbf{delegability frontier} (quality vs.\ autonomy) advances with self-improvement. We also prove a theorem that AAI-3 agent becomes AAI-5 over time with sufficient conditions, formalizing "baby AGI" becomes Superintelligence intuition.
\end{abstract}

\textbf{Keywords:} autonomous agents, AGI, self-improvement, Kardashev scale, AI, superintelligence

\section{Introduction}
Scaled automation is moving from brittle scripts (RPA) to open-ended agentic systems that plan, use tools, collaborate, and adapt. Researchers and practitioners often speak in levels (e.g., ANI$\to$AGI$\to$ASI; lab-specific grades of ``agents,'' ``reasoners,'' or ``organizations''), but such ladders are descriptive, not operational. They neither quantify how autonomous a system is, nor whether it can improve itself without human supervision.

We propose an \textbf{Autonomous AI (AAI) Scale} that: (i) is multi-dimensional; (ii) yields a single composite score for comparison; (iii) adds a temporal axis via $\kappa$ for self-improvement; and (iv) requires closure properties to establish sustainable autonomy. Inspired by Kardashev's civilizational scale, our AAI-0\ldots AAI-4 progression names intuitive phases while grounding each level in measurable gates.

This formalizes definitions of AGI and superintelligence, focusing on their self-improvement rates. We also prove a theorem that an AAI-3 agent becomes AAI-5 over time with sufficient conditions. This formalizes the intuition that sufficiently intelligent autonomous agents will become superintelligent (become better than humans at any task). 

Final remark, especially comparing with \cite{hendrycks-agi}, our definition of AGI is operational and results-driven rather than based on cognition. We find human cognition models, though useful in guidance, not sufficient as AGI itself is usually defined through results. Approach to AGI through swarm intelligence, coordinated effort of a legion of agents, is more likely to yield positive results than a single agent/model approach. This would also mean that we should rather be looking at society- or civilization-like intelligence/organization to get more insights into progress of intelligence. We compare our definitions with \cite{hendrycks-agi} in Section \ref{sec:chc-alignment}.

\subsection*{Contributions}
\begin{enumerate}[itemsep=0.25em, topsep=0.25em]
  \item A multi-axis operationalization of autonomy/general intelligence with ten normalized axes and a composite \textbf{AAI-Index}.
  \item \textbf{Self-Improvement Coefficient} ($\kappa$) and \textbf{closure} properties (maintenance \& expansion) for falsifiable ``self-improving AI.''
  \item \textbf{OWA-Bench}, a benchmark suite that stresses long-horizon planning, persistence under drift, tool discovery, and multi-agent coordination.
  \item \textbf{Level gates} (AAI-0\ldots AAI-4) with concrete thresholds enabling reproducible placement and progress tracking.
  \item Proof that any AAI-3 agent improving fast enough will eventually become AAI-5 (\textbf{escape rate of self-improvement}).
\end{enumerate}

\section{Related Work}

Measuring and comparing autonomy and general capability in AI has evolved along five complementary lines: (i) conceptual capability ladders that outline qualitative phase transitions in competence; (ii) matrix-style Levels of AGI frameworks that separate breadth (generality) from depth (relative performance); (iii) enterprise-oriented agentic automation maturity models; (iv) cosmological analogies (notably Kardashev) that offer a memorable long-run narrative; and (v) empirical benchmarks for agency probing tool use, long-horizon control, coordination, and sim-to-real. Our AAI scale leverages these strands while addressing three persistent gaps: (a) the need for a multi-axis composite that penalizes lopsided systems, (b) a longitudinal measure of self-improvement, and (c) auditable closure tests under drift and expansion.

\paragraph{Capability ladders.}
Popular ladders sketch qualitative progressions from narrow assistants to increasingly agentic systems. These narratives are useful for communication but typically lack operational gates, measurement-error treatment, or anti-gaming controls \cite{em360-2019-7-stages,floudas-2024-taxonomy}. We retain the intuitive staged arc while replacing prose-only gates with measurable axes and thresholds.

\paragraph{Levels of AGI frameworks.}
Matrix-style proposals explicitly distinguish breadth (coverage of domains) from depth (relative human parity), producing cell-based levels that can be audited \cite{morris-2023-levels}. Public-facing summaries have broadened adoption \cite{rosenberg-2024-axios-agi-levels}. These frameworks sharpen terminology, yet often omit longitudinal self-improvement rates, closure proofs under drift, and economics-aware throughput.

\paragraph{Agentic automation maturity (enterprise focus).}
Industry roadmaps describe a transition from RPA to agentic assistants and onward to autonomous orchestration of end-to-end workflows, emphasizing planning, tool use, and continuous improvement \cite{karjalainen-2024-agentic-automation}. Such models are pragmatic but rely on qualitative checklists; our AAI scale replaces these with standardized axes, a longitudinal coefficient for self-improvement, and auditable maintenance/expansion closures.

\paragraph{Cosmological analogies (Kardashev-style).}
Kardashev's typology ranks civilizations by extractable energy; AI analogies borrow this mnemonic arc to reason about capability scope \cite{kardashev-wikipedia,campbell-2017-kardashev-ai}. We adopt the mnemonic (AAI-0\,…\,AAI-4) while grounding each level in quantitative, reproducible gates.

\paragraph{Benchmarks for agency (web, code, multi-agent, embodied).}
Open-world web evaluations stress tool use, planning, and robustness to UI/auth/schema drift \cite{zhou-2024-webarena,yoran-2024-assistantbench}. Software-engineering benchmarks probe persistent reasoning and tool-mediated edits on real repositories \cite{swebench-2023}. Multi-agent frameworks analyze role specialization and coordination benefits \cite{wu-2023-autogen,li-2023-camel}. Embodied-AI platforms quantify sim-to-real transfer, safety, and control quality \cite{szot-2021-habitat2,li-2022-igibson2,tobin-2017-domain-rand,brohan-2023-rt2,liu-2022-behavior-habitat2}. Complementary lines study self-improvement via reflection, revision, and open-ended exploration \cite{shinn-2023-reflexion,huang-2023-voyager}. These inform our axis design (Autonomy, Planning, Tool Economy, Sociality, Embodiment) and motivate a longitudinal coefficient for auditable self-improvement.

\paragraph{Synthesis and gap.}
Across ladders \cite{em360-2019-7-stages,floudas-2024-taxonomy}, levels \cite{morris-2023-levels,rosenberg-2024-axios-agi-levels}, maturity models \cite{karjalainen-2024-agentic-automation}, analogies \cite{kardashev-wikipedia,campbell-2017-kardashev-ai}, and benchmarks \cite{zhou-2024-webarena,yoran-2024-assistantbench,swebench-2023,wu-2023-autogen,li-2023-camel,szot-2021-habitat2,li-2022-igibson2,tobin-2017-domain-rand,brohan-2023-rt2,liu-2022-behavior-habitat2,shinn-2023-reflexion,huang-2023-voyager}, three gaps remain: (i) absence of a multi-axis composite discouraging over-optimization to a single suite, (ii) missing longitudinal rate metrics of self-improvement, and (iii) lack of closure proofs for maintenance under drift and autonomous expansion. The AAI scale directly addresses these. 

The most recent and conceptually integrated effort to define AGI is in \cite{hendrycks-agi}, which grounds “general intelligence” in the Cattell-Horn-Carroll (CHC) framework. In this view, AGI is not a monolith but a breadth-and-proficiency profile across broad abilities (e.g., knowledge/comprehension, reading-writing, fluid reasoning, working memory, long-term storage and retrieval, visual and auditory processing, processing speed, quantitative knowledge), each evaluated against well-educated adult human norms. The resulting score is non-compensatory in spirit: severe deficits in core cognitive functions cannot be masked by strength elsewhere, and assessments emphasize reliability, long-horizon retention, and retrieval fidelity rather than surface heuristics.

\paragraph{Relation to CHC.}
Our AAI scale complements \cite{hendrycks-agi} by focusing on agentic performance under deployment constraints - planning with tools, persistence, self-revision, multi-agent coordination, embodiment/safety, and throughput economics. Whereas the CHC-based definition characterizes a system's cognitive endowment, AAI characterizes what the system can reliably do in open-ended settings. This complementarity is made explicit in \S\ref{sec:chc-alignment}, where we provide a crosswalk from AAI axes (Generality $G$, Planning $P$, Memory/Persistence $M$, World-Model Fidelity $W$, Autonomy $A$, Tool Economy $T$, Self-Revision $R$, Sociality $S$, Embodiment $E$, Economic Throughput $\text{\$}$) to the closest CHC domains and import CHC-style probes into OWA-Bench.

In brief, we (i) anchor cognitively aligned axes to human performance following \cite{hendrycks-agi} and require minimum cognitive gates on memory, retrieval fidelity, and reasoning before promotion; (ii) aggregate evidence with a non-compensatory composite that penalizes lopsided profiles rather than rewarding a single strong axis; and (iii) track change over time with a simple longitudinal term and auditable checks that improvements persist under interface drift and do not regress prior skills. Formal definitions, thresholds, and the aggregation procedure are deferred to \S\ref{sec:chc-alignment}.

\paragraph{Practical synthesis.}
In sum, \cite{hendrycks-agi} offers a principled, human-referenced cognitive backbone; AAI supplies the agentic, tool-using, and economics-aware shell. Together they yield a two-view evaluation: a system should (1) meet CHC-derived minima on the cognitive core to qualify for delegated autonomy, and (2) demonstrate stable, audited performance on AAI's deployment axes to earn higher autonomy levels. We operationalize this synthesis in \S\ref{sec:chc-alignment} via (i) a CHC$\leftrightarrow$AAI crosswalk and anchors, (ii) CHC-style micro-batteries embedded in OWA-Bench, (iii) level gates tied to memory and retrieval fidelity, and (iv) a jaggedness penalty ensuring breadth with balance.

\paragraph{Representation theory.}
We take a representation-theoretic view of autonomous systems. An agent is never observed “directly,” but through a battery of tasks that induces a structured representation of its behavior: logs, traces, artifacts, and scores that are invariant to irrelevant implementation details and equivariant to controlled changes in the environment. 

The AAI scale, OWA-Bench suites, and their gates define the map from agents to representations; the human anchors and level criteria define which features of that representation matter. Discrete levels act as equivalence classes over representations that satisfy the same auditable properties, while continuous telemetry (the link score and the Delegability Frontier) summarizes dynamics within a class. In this framing, claims about “human-level” ability are statements about the alignment between an agent's representation and human reference representations under the same measurement process. 

The goal is not to capture every internal mechanism, but to enforce sufficiency, identifiability, and robustness of what is measured, so that improvements in the representation correspond to genuine, externally valid improvements in capability.

These representation theoretic insights will be formalized in the next paper.

\section{The AAI Scale: Axes, Index, and Levels}
\label{sec:aai}
A \emph{battery} is the tuple
\[
\mathcal{B}\;=\;\big(T,\;\mathcal{F},\;\mathsf{S},\;Q^{*},\;\mu,\;\mathsf{D},\;\Pi,\;\mathsf{R}\big),
\]
where:
(i) $T$ is a finite set of task \emph{instances} $t\in T$ (e.g., webpages, APIs, repos, simulated scenes) with associated initial context and termination conditions; 
(ii) $\mathcal{F}=\{F_1,\ldots,F_N\}$ is a partition of $T$ into \emph{families} (domains) used for generality accounting; 
(iii) $\mathsf{S}=\{S_t:\Omega_t\to[0,1]\}_{t\in T}$ are task-specific \emph{proper scoring rules} that map an execution trace $\omega\in\Omega_t$ to a quality score $q(t)=S_t(\omega)$; 
(iv) $Q^{*}:T\to(0,1)$ sets task-level target quality thresholds $q^{*}(t)$; 
(v) $\mu$ is a sampling measure over triples $(t,s,\delta)\in T\times \Pi\times \mathsf{D}$ specifying how tasks, random seeds $s$ (for environment/randomization), and \emph{drifts} $\delta$ (UI/schema/layout perturbations) are drawn; 
(vi) $\mathsf{D}$ is a family of drift operators acting on tasks/environments, with a prescribed schedule or magnitude law (e.g., $\delta\sim\mathcal{N}(0,\sigma^2)$ for parametric perturbations or a discrete set of DOM changes); 
(vii) $\Pi$ is a seed space (RNG states) used for reproducibility and Confidence Intervals (CI) bands; 
(viii) $\mathsf{R}$ is a \emph{resource accounting} schema mapping action/tool/compute events in a trace to nonnegative costs $c(t)\in\mathbb{R}_{\ge 0}$ (e.g., API spend, GPU time, operator minutes). 

Given an agent $\mathcal{A}$ and budgetary constraints from $\mathsf{R}$, an evaluation draws i.i.d.\ samples $(t_i,s_i,\delta_i)\sim\mu$ for $i=1,\dots,n$ and produces traces 
\[
\omega_i \;=\; \mathrm{Run}\!\left(\mathcal{A};\,t_i,\,s_i,\,\delta_i,\,\mathsf{R}\right),
\]
from which we derive per-task observables: quality $q(t_i)=S_{t_i}(\omega_i)$, success $z(t_i)=\mathbb{I}\{q(t_i)\ge q^{*}(t_i)\}$, uninterrupted action count $a(t_i)\in\mathbb{N}$, plan depth $d(t_i)\in\mathbb{N}$ (the length of the longest executed path of prerequisite actions) and incurred cost $c(t_i)\in\mathbb{R}_{\ge 0}$. Axis-specific raw statistics $r_x=r_x(\{\omega_i\}_{i=1}^n)$ are computed by fixed functionals (see Eqs.\,(\ref{eq:A})-(\ref{eq:dollar}) for $A,\ldots,\text{\$})$ and then normalized by calibration maps $\phi_x$ in \eqref{eq:phi} to yield $x\in[0,1]$. Family means $\bar q(F_j)=|F_j|^{-1}\sum_{t\in F_j}q(t)$ and coverage indicators $\mathbb{I}\{\bar q(F_j)\ge \tau_j\}$ support generality $G$ as in \eqref{eq:G}. Uncertainty is obtained by resampling the evaluation law $\mu$ (bootstrap or block-bootstrap under temporal correlation), yielding CIs for each axis and for the composite index $\mathcal{C}$ via \eqref{eq:index}. 

\paragraph{Capability (base and aggregate).}
Given the battery $\mathcal{B}=(T,\mathcal{F},\mathsf{S},Q^{*},\mu,\mathsf{D},\Pi,\mathsf{R})$
and traces $\omega_i=\mathrm{Run}(\mathcal{A};t_i,s_i,\delta_i,\mathsf{R})$, the
\emph{base per-instance capability} is the task quality already defined by the
proper scoring rule:
\[
q(t)\;=\;S_t(\omega)\in[0,1].
\]
Optionally, a strict-success variant uses $z(t)=\mathbb{I}\{q(t)\ge q^{*}(t)\}$.
For a fixed instance $t$, the \emph{per-task capability} of agent $\mathcal{A}$
under the evaluation law is the seed/drift expectation
\[
C(\mathcal{A};t)\;=\;\mathbb{E}_{(s,\delta)\sim \mu(\cdot\,|\,t)}
\Big[\,S_t\!\big(\mathrm{Run}(\mathcal{A};t,s,\delta,\mathsf{R})\big)\,\Big],
\]
estimated by the sample mean over draws with $t_i=t$.
Aggregating over instances with battery-fixed weights $\{\omega_t\}_{t\in T}$,
(e.g., uniform within families with family weights that sum to $1$),
yields the \emph{aggregate capability}
\[
\mathcal{C}(\mathcal{A})\;=\;\sum_{t\in T}\omega_t\,C(\mathcal{A};t)
\quad\text{(or }\ \widehat{\mathcal{C}}(\mathcal{A})=\sum_{i=1}^{n}\omega_{t_i}\,q(t_i)\ \text{ in-sample)}.
\]
This base capability is used wherever a task-level score is needed (e.g., for solo
vs.\ multi-role comparisons) and is conceptually separate from axis-specific
functionals $r_x(\{\omega_i\})$ that also consume other observables (plan depth,
cost, etc.) before normalization by $\phi_x$.

\emph{Admissibility.} A battery $\mathcal{B}$ is \emph{admissible} if: (a) each $S_t$ is a proper (or strictly proper) scoring rule on its outcome space; (b) every family $F_j$ satisfies $|F_j|\ge m_{\min}$ for some fixed $m_{\min}\ge 5$; (c) $\mu$ has full support on all intended drift magnitudes used to score $T$; (d) resource accounting $\mathsf{R}$ is fixed during a leaderboard window; and (e) seed disclosure and artifact escrow (snapshots, diffs, checksums) ensure replicability. 

Requiring admissibility makes results statistically sound, hard to game, and comparable over time: (a) proper (or strictly proper) scoring rules ensure systems are rewarded exactly for calibrated performance---no incentive to hedge or exploit quirks---so confidence intervals and risk metrics are meaningful; (b) a minimum per-family size $m_{\min}$ gives each domain enough samples for power, stable coverage indicators, and reliable CI bands, avoiding ``one-shot'' overfitting; (c) full-support drift sampling guarantees that tool-use and robustness scores (e.g., $T$) reflect performance across the intended spectrum of UI/API changes rather than a cherry-picked subset; (d) fixed resource accounting $\mathsf{R}$ during a leaderboard window makes cost/latency/throughput comparisons fair and repeatable, preventing budget inflation from masquerading as capability; and (e) seed disclosure plus artifact escrow (snapshots, diffs, checksums) delivers replicability and auditability, enabling independent re-runs, longitudinal tracking, and governance uses (e.g., verifying no regressions or hidden configuration changes).

\emph{Remark.} In persistence settings (memory $M$), $t$ expands to $(t,\Delta)$ where $\Delta$ is an inter-session lag; the scoring rule $S_{t,\Delta}$ induces retention curves $q(\Delta)$ whose summaries (half-life or decay rate) enter $M$ before normalization.

\emph{Provenance.} Our tuple formalization extends the psychometric notion of a \emph{test battery} \cite{apa-battery,sage-battery} and ML benchmark suites (e.g., GLUE) \cite{wang-2019-glue}, while making explicit proper scoring \cite{gneiting-raftery-2007-proper}, concept drift \cite{gama-2014-concept-drift}, and bootstrap uncertainty \cite{efron-tibshirani-1993-bootstrap}.

\subsection{Axes (normalized to [0,1])}
Let $\mathcal{X}=\{A,G,P,M,T,R,S,E,W,\text{\$}\}$ denote ten axes. Each $x\in\mathcal{X}$ is a \emph{normalized} functional of raw task outcomes on a domain-specific battery $\mathcal{B}$ with fixed scoring rules that we define below. We use anchors $L_x<U_x$ to define a monotone calibration map

\begin{equation}
\phi_x(r)\;=\;\mathrm{clip}\!\left(\frac{r-L_x}{U_x-L_x},\,0,\,1\right),
\qquad x\in\mathcal{X},
\label{eq:phi}
\end{equation}
that maps a raw statistic $r$ into $[0,1]$. Unless specified, $L_x$ is a strong-but-feasible baseline (e.g., tuned single-agent LLM) and $U_x$ a reference target (e.g., expert-team performance, or a challenge-set ceiling). Uncertainty is summarized by nonparametric bootstrap CIs over tasks/seeds.

\paragraph{Notation (battery-level).}
For a task $t\in\mathcal{B}$, let $q(t)\in[0,1]$ be a quality score from the task's proper scoring rule; $a(t)\in\mathbb{N}$ the number of uninterrupted actions; $d(t)\in\mathbb{N}$ the minimal solved plan depth; $c(t)\in\mathbb{R}_{\ge 0}$ the incurred monetary cost; and $z(t)\in\{0,1\}$ a success indicator at the task's target threshold $q^{*}(t)\in(0,1)$. Let $\mathbb{I}\{\cdot\}$ be the indicator. All expectations below are empirical means over $\mathcal{B}$ unless stated.

\vspace{0.35em}
\noindent\textbf{A - Autonomy.}
With a horizon cap $H\in\mathbb{N}$,
\begin{equation}
\widehat{A}\;=\;\frac{1}{|\mathcal{B}|}\sum_{t\in\mathcal{B}}\min\!\left(\frac{a(t)}{H},\,1\right),\qquad
A\;=\;\phi_A\!\big(\widehat{A}\big).
\label{eq:A}
\end{equation}
\emph{Remarks.} $H$ prevents inflation by ultra-long runs; choose $H$ per domain (e.g., UI workflows vs.\ long builds). Report the distribution of $a(t)$ and the fraction of tasks completed without human intervention.

\vspace{0.25em}
\noindent\textbf{G - Generality.}
Partition the battery into families $\mathcal{F}=\{F_1,\ldots,F_N\}$; declare family $F_i$ \emph{covered} if the mean quality meets its threshold: $\bar q(F_i)\!=\!\frac{1}{|F_i|}\sum_{t\in F_i}q(t)\ge \tau_i$. Then
\begin{equation}
\widehat{G}\;=\;\frac{1}{N}\sum_{i=1}^{N}\mathbb{I}\{\bar q(F_i)\ge\tau_i\},\qquad
G\;=\;\phi_G\!\big(\widehat{G}\big).
\label{eq:G}
\end{equation}
\emph{Remarks.} Choose $\tau_i$ as a relative percentile vs.\ vetted baselines or a fixed KPI. Family weighting can be uniform (default) or application-specific.

\vspace{0.25em}
\noindent\textbf{P - Planning depth.}
Let $D$ be the target depth anchor (per suite). Define
\begin{equation}
\widehat{P}\;=\;\frac{1}{|\mathcal{B}|}\sum_{t\in\mathcal{B}}\min\!\left(\frac{d(t)}{D},\,1\right),\qquad
P\;=\;\phi_P\!\big(\widehat{P}\big).
\label{eq:P}
\end{equation}
\emph{Remarks.} Use the smallest depth at which the agent solves the task (credit partial hierarchical plans). For tasks without explicit hierarchies, proxy with minimal dependency-chain length.

\vspace{0.25em}
\noindent\textbf{M - Memory/Persistence.}
For each task family with persistence, measure a retention curve $q(\Delta)$ at lag $\Delta$ (days), and fit an exponential $q(\Delta)\approx q_0 e^{-\lambda\Delta}$ or compute the empirical half-life $t_{1/2}=\inf\{\Delta:\,q(\Delta)\le q_0/2\}$ which is related by $t_{1/2}=\frac{\ln 2}{\lambda}$. Combine retention and retrieval:
\begin{align}
\widehat{M}_1&=\exp(-\lambda/\lambda_{\max}),\qquad
\widehat{M}_2=\mathrm{Rec}@K,\\
\widehat{M}&=\tfrac{1}{2}(\widehat{M}_1+\widehat{M}_2),\qquad
M=\phi_M(\widehat{M}).
\label{eq:M}
\end{align}
\emph{Remarks.} $\lambda_{\max}$ sets a tolerable forgetting rate (can be chosen by picking a minimum acceptable half-life $t_{\min}$); $\mathrm{Rec}@K$ is retrieval recall@K on logged facts/specs. Use the median across families. Interpretation: $\lambda=0$ (infinite half-life) $\Rightarrow \widehat M_1=1$; $t_{1/2}=t_{\min}\Rightarrow \widehat M_1=e^{-1}\approx0.37$;
$t_{1/2}=2t_{\min}\Rightarrow \widehat M_1\simeq e^{-1/2}\approx0.61$.

\vspace{0.25em}
\noindent\textbf{T - Tool Economy.}
Let $\mathcal{T}$ be \emph{tool categories} (APIs, DBs, IDE actions, browsers, cloud ops) discovered/used with success under drift. Define coverage and drifted success:
\begin{align}
\text{cov} &= \frac{|\{c\in\mathcal{T}:\,\text{used\&required}\}|}{|\{c\in\mathcal{T}:\,\text{required}\}|},\\
\text{succ}_\delta &= \mathbb{E}\big[z_\delta(t)\big],\quad z_\delta(t)=\mathbb{I}\{q_\delta(t)\ge q^{*}(t)\},
\end{align}
where $q_\delta$ scores under controlled schema/UI drifts of magnitude $\delta$. With a size prior $S=\log(1+|\mathcal{T}|)/\log(1+S_{\max})$,
\begin{equation}
\widehat{T}=\big(\text{cov}\cdot \text{succ}_\delta\cdot S\big)^{1/3},\qquad T=\phi_T(\widehat{T}).
\label{eq:T}
\end{equation}
\emph{Remarks.} $S$ rewards breadth with diminishing returns; report results for multiple $\delta$.

\vspace{0.25em}
\noindent\textbf{R - Self-Revision.}
Consider an autonomous revision event $r$ that modifies model/policy/workflow without human labels occurring between times $t_{\mathrm{pre}}$ and $t_{\mathrm{post}}$. Let $S_{\mathrm{rev}}^{\mathrm{pre}}$ and $S_{\mathrm{rev}}^{\mathrm{post}}$ be the system before/after $r$;
let $S_{\mathrm{ctrl}}$ be a frozen control. Evaluate a capability metric
$C(\cdot\,;\mathcal{H}_r)$ on a matched holdout $\mathcal{H}_r$ (same weights, seeds, and resource budget):

\[
\widehat{C}^{\mathrm{pre}}_{\mathrm{rev}} \;=\; C\!\left(S_{\mathrm{rev}}^{\mathrm{pre}};\mathcal{H}_r\right),\quad
\widehat{C}^{\mathrm{post}}_{\mathrm{rev}} \;=\; C\!\left(S_{\mathrm{rev}}^{\mathrm{post}};\mathcal{H}_r\right),\quad
\widehat{C}^{\mathrm{pre}}_{\mathrm{ctrl}} \;=\; C\!\left(S_{\mathrm{ctrl}};\mathcal{H}_r\right),\quad
\widehat{C}^{\mathrm{post}}_{\mathrm{ctrl}} \;=\; C\!\left(S_{\mathrm{ctrl}};\mathcal{H}_r\right).
\]

The \emph{debiased capability delta} (difference-in-differences) is defined by
\[
\Delta\mathcal{C}_r \;=\;
\big(\widehat{C}^{\mathrm{post}}_{\mathrm{rev}}-\widehat{C}^{\mathrm{pre}}_{\mathrm{rev}}\big)\;-\;
\big(\widehat{C}^{\mathrm{post}}_{\mathrm{ctrl}}-\widehat{C}^{\mathrm{pre}}_{\mathrm{ctrl}}\big).
\]

To quantify how much of the propose$\to$implement$\to$validate pipeline was autonomous, define stage-level autonomy fractions
$a^{(P)}_r,a^{(I)}_r,a^{(V)}_r\in[0,1]$ for \textbf{propose}, \textbf{implement}, and \textbf{validate}, respectively.
Let $(\alpha_P,\alpha_I,\alpha_V)\!\in\![0,1]^3$ with $\alpha_P{+}\alpha_I{+}\alpha_V{=}1$ be fixed stage weights.
Then the \emph{autonomy factor} is
\[
\rho_r \;=\; \alpha_P\,a^{(P)}_r \;+\; \alpha_I\,a^{(I)}_r \;+\; \alpha_V\,a^{(V)}_r \;\in\; [0,1].
\]

\textit{Default instantiation.} $a^{(P)}_r{=}1$ if the change was agent-initiated (no human prompt specifying the patch), else $0$;
$a^{(I)}_r$ is the fraction of applied edits/actions executed automatically (e.g., auto-generated diff lines over total);
$a^{(V)}_r$ is the fraction of acceptance checks that were fully automated (no human labels), e.g., passed regression tests / total checks.
Use $(\alpha_P,\alpha_I,\alpha_V)=(\tfrac{1}{3},\tfrac{1}{3},\tfrac{1}{3})$ unless otherwise stated.

We define Self Revision by aggregating over events in a window $\mathcal{R}_{\text{win}}$ (the set of self-revision events that occurred within a fixed reporting time window, and pass filters (matched pre/post eval on the same holdout, resource parity, logged diffs/seeds/checksums, no disqualifying human labels)):
\begin{equation}
\widehat{R}\;=\;\mathrm{clip}\!\left(\frac{1}{Z}\sum_{r\in\mathcal{R}_{\text{win}}}\rho_r\cdot \max(\Delta \mathcal{C}_r,0),\,0,\,1\right),\quad
R=\phi_R(\widehat{R}),
\label{eq:R}
\end{equation}
with $Z$ a scaling anchor (e.g., target cumulative gain). \emph{Remarks.} We require ablations to establish causality - e.g., temporarily revert the patch or disable the discovered tool; the gain should vanish if the cause is removed. We also log diffs, seeds, and checksums so runs are reproducible and auditable (same code/data config recreates the measured $\Delta\mathcal{C}_r$).

\textit{Example:} The agent introduces a new indexer; VRP@K on the holdout moves $0.78\!\to\!0.84$ while the frozen control moves $0.78\!\to\!0.80$, so $\Delta\mathcal{C}_r=(0.84-0.78)-(0.80-0.78)=0.04$. If the agent proposed, implemented, and validated with minimal human approval, set $\rho_r=0.9$. With $Z=0.10$, this event contributes $0.9\!\times\!0.04=0.036$ to the sum (before scaling and clipping).

\vspace{0.25em}
\noindent\textbf{S - Sociality/Coordination.}
Our goal is to measure the \emph{coordination premium} unlocked by allowing multiple concurrent roles/agents \emph{under identical tasks, tools, and total resource budgets}. The single agent is treated as the special case $m{=}1$ of the same system.

Let $\mathcal{A}_m$ denote the system with a \emph{concurrency cap} of $m\in\{1,\ldots,M\}$ active roles/agents; all other settings (tasks, tools/APIs, resource accounting $\mathsf{R}$, seeds/drifts sampling law $\mu$) are held fixed. For each task instance $t\in T$, define the per-task capability $C_m(t)\in[0,1]$ by reusing the battery's proper scoring rules $S_t$:
\[
C_m(t)\;=\;\mathbb{E}_{(s,\delta)\sim \mu(\cdot\,|\,t)}\!\left[\,S_t\!\big(\mathrm{Run}(\mathcal{A}_m;\,t,s,\delta,\mathsf{R})\big)\right]
\quad(\text{estimated by the seed/drift average}).
\]

For each $t$, define the headroom-normalized lift from $m{=}1$ to the best multi-role configuration:
\[
s_t \;=\; \left[\,\frac{\displaystyle \max_{m>1\le M} C_m(t)\;-\;C_1(t)}{\,1 - C_1(t) + \varepsilon\,}\,\right]_+,
\qquad [x]_+=\max(x,0),\ \ \varepsilon\in(0,10^{-6}],
\]
which measures how much of the remaining headroom over $C_1(t)$ is realized via coordination (clipped at zero if coordination does not help). Aggregate across the battery with fixed weights $\{\omega_t\}_{t\in T}$, $\sum_t\omega_t=1$:
\[
\bar s \;=\; \sum_{t\in T}\omega_t\, s_t.
\]

\emph{Deadlock/chatter penalty.} Let $m^\star(t)\in\arg\max_{m>1} C_m(t)$ (tie-break by lower comms-per-action). For each $(t,s,\delta)$ evaluated at $m^\star(t)$, form an \emph{episode} $e$. From execution logs define indicators:
\begin{itemize}
  \item $D_e\in\{0,1\}$: \textbf{unresolved conflict} - no final commit or incompatible commits at cutoff.
  \item $L_e\in\{0,1\}$: \textbf{loop} - no new verified external action for $\ge L$ consecutive turns or a repeated tool-call cycle.
  \item $C_e\in\{0,1\}$: \textbf{chatter overflow} - communication tokens per verified external action exceed threshold $\tau_{\text{comm}}$.
  \item $M_e\in\{0,1\}$: \textbf{mode collapse} - near-duplicate messages or identical tool calls for $\ge K$ successive turns without progress.
\end{itemize}
Let $\overline{X}=\frac{1}{|\mathcal{E}|}\sum_{e\in\mathcal{E}}X_e$ be episode means and define
\[
\pi_{\text{conflict}} \;=\; \overline{\mathbb{I}\{D_e \lor L_e\}},
\qquad
\pi_{\text{deadlock}} \;=\; \mathrm{clip}\!\big(\pi_{\text{conflict}} + \alpha\,\overline{C_e} + \beta\,\overline{M_e},\,0,\,1\big),
\]
with small weights $\alpha,\beta\in[0,1]$ (defaults $\alpha{=}\beta{=}0.25$) and defaults $L{=}3$, $K{=}3$, and $\tau_{\text{comm}}$ set to the $m{=}1$ median communications-per-action on $T$ (a \emph{verified external action} is a tool/API call, code run, or environment transition validated by the task oracle).

\emph{Score and normalization.} The raw sociality statistic is the penalized lift
\[
\widehat{S} \;=\; \mathrm{clip}\!\big(\bar s\cdot(1-\pi_{\text{deadlock}}),\,0,\,1\big),
\]
and the reported axis score is $S=\phi_S(\widehat{S})$ via the calibration map in~\eqref{eq:phi}.

\emph{Fairness and comparability.} All comparisons across $m$ use the \emph{same} tasks $T$, tools/APIs, and a \emph{fixed total resource budget} per task (time/tokens/API calls) under $\mathsf{R}$; seeds/drifts are matched or resampled consistently from $\mu$. The only degree of freedom is \emph{concurrency} (number of roles and protocol), so $S$ isolates the coordination premium rather than resource scaling. Uncertainty for $S$ follows Sec.~\ref{sec:aai}: nonparametric bootstrap over tasks and seeds (or block-bootstrap if temporally correlated), yielding CIs.

\vspace{0.25em}
\noindent\textbf{E   Embodiment/Actuation (robotics only).}
This axis measures \emph{whether a robotic agent can act reliably and safely in the real world and whether simulator performance transfers to reality}.
We combine three terms-\emph{actuation reliability} (AR), \emph{safety} (SS), and \emph{sim-to-real agreement} (S2R)-in a non-compensatory way.

\emph{Evaluation setting.} Let $\mathcal{H}_{\mathrm{real}}$ be the set of \emph{real-world} episodes produced by draws $(t,s,\delta)$ under the battery's law $\mu$ with the physical robot and resource schema $\mathsf{R}$; let $\mathcal{H}_{\mathrm{sim}}$ be the matched set of \emph{simulation} episodes for the same task instances (or paired surrogates) using the same policies and domain-randomization schedule $\mathsf{D}$. For any episode, the battery already defines quality $q(t)=S_t(\omega)$ and success $z(t)=\mathbb{I}\{q(t)\ge q^{*}(t)\}$ (Sec.~\ref{sec:aai}); safety-stopped or out-of-bounds runs count as $z(t)=0$.

\emph{Actuation reliability (AR).} AR is the expected strict success rate on real executions:
\[
\mathrm{AR}\;=\;\mathbb{E}_{(t,s,\delta)\in\mathcal{H}_{\mathrm{real}}}\big[z(t)\big]
\;\;\approx\;\;\frac{1}{|\mathcal{H}_{\mathrm{real}}|}\sum_{(t,s,\delta)\in\mathcal{H}_{\mathrm{real}}} z(t)\in[0,1].
\]
For reporting, include bootstrap CIs over tasks/seeds.

\emph{Safety score (SS).} Incidents are categorized by severity: \emph{near-miss} (nm), \emph{minor} (min), \emph{major} (maj), \emph{critical} (crit).
Let $H$ be the total exposure hours in $\mathcal{H}_{\mathrm{real}}$. Define \emph{rates per 100~hours}
\[
\nu_{x}\;=\;100\,\frac{N_x}{H},\quad x\in\{\text{nm,min,maj,crit}\},
\]
where $N_x$ is the count of $x$-incidents. Example taxonomy: near-miss = protective stop or breach of a distance/force margin; minor = contact or property damage below reporting threshold; major = reportable injury or significant damage; critical = severe injury, uncontrolled hazard, or standards breach.
Choose nonnegative severity weights $w_{\text{nm}},w_{\text{min}},w_{\text{maj}},w_{\text{crit}}$ (defaults $0.25,1,5,20$).
Define
\[
\mathrm{SS}\;=\;1-\min\!\left(1,\,w_{\text{nm}}\nu_{\text{nm}}+w_{\text{min}}\nu_{\text{min}}+w_{\text{maj}}\nu_{\text{maj}}+w_{\text{crit}}\nu_{\text{crit}}\right)\in[0,1].
\]
If $N_{\text{crit}}>0$ in the reporting window, set $\mathrm{SS}=0$ (and report the incident), regardless of the weighted sum.

\emph{Sim-to-real agreement (S2R).} Let $\mathrm{AR}_{\text{sim}}$ be the strict success rate computed on the matched simulation episodes $\mathcal{H}_{\mathrm{sim}}$ (same tasks/policies, domain randomization from $\mathsf{D}$), and let $\mathrm{AR}_{\text{real}}=\mathrm{AR}$ from above.
Define
\[
\mathrm{S2R}\;=\;1-\big|\mathrm{AR}_{\text{sim}}-\mathrm{AR}_{\text{real}}\big|\in[0,1],
\]
so perfect transfer gives $\mathrm{S2R}=1$, and large simulation–reality gaps reduce the score. (If the matching is imperfect, reweight sim episodes to the real context distribution before computing $\mathrm{AR}_{\text{sim}}$.)

\emph{Composite and reporting.} Combine the three terms with a geometric mean to prevent compensation:
\[
\widehat{E}\;=\;(\mathrm{AR}\cdot \mathrm{SS}\cdot \mathrm{S2R})^{1/3},\qquad E=\phi_E(\widehat{E}).
\]
Report \emph{MTBF} (mean time between failures = $H/N_{\text{fail}}$, failures include $z(t)=0$ due to actuation), \emph{MTTR} (mean time to repair), and \emph{MTBSI} (mean time between safety incidents = $H/(N_{\text{nm}}{+}N_{\text{min}}{+}N_{\text{maj}}{+}N_{\text{crit}})$), together with Poisson/bootstrap CIs for incident rates. The geometric mean and the critical-incident gate enforce that high task success cannot offset unsafe behavior or poor transfer.

\vspace{0.25em}
\noindent\textbf{W   World-Model Fidelity.}

This axis measures whether the agent assigns accurate, well-calibrated probabilities to verifiable propositions.

For a sampled triple $(t,s,\delta)\sim\mu$, the evaluation produces a finite trace
$\omega_{t,s,\delta}=(r_1,\ldots,r_{N(\omega)})\in\Omega_t$ with a terminal record $r_{N(\omega)}$
(at termination). The terminal record contains a designated scalar coordinate
$\mathrm{prob}\in[0,1]$ giving the agent's stated probability for the proposition evaluated by $t$.
Define the per-run probability and binary target by
\[
p_{t,s,\delta}\;\equiv\;\big(r_{N(\omega_{t,s,\delta})}\big)_{\mathrm{prob}}\in[0,1]
\]
\[
y_{t,s,\delta}\;\equiv\;\mathbb{I}\{\text{the proposition for }t\text{ is true in the environment with }(s,\delta)\}\in\{0,1\}.
\]
If the terminal record stores odds $o\!\ge\!0$ or an interval $[a,b]\!\subset\![0,1]$ in lieu of a scalar, set
$(r_{N(\omega)})_{\mathrm{prob}}:=o/(1{+}o)$ or $(a{+}b)/2$, and clip to $[0,1]$.
(When truth does not depend on $(s,\delta)$, we write $y_t$ for brevity.)

We score probabilistic truthfulness by the Brier risk under the evaluation law,
\[
\mathrm{Brier}\;=\;\mathbb{E}_{(t,s,\delta)\sim\mu}\!\big[(p_{t,s,\delta}-y_{t,s,\delta})^2\big]
\]
and compare it to a pre-registered reference predictor $\pi_{\text{ref}}:T\to[0,1]$ (e.g., marginal or retrieval-only), whose Brier is
\[
\mathrm{Brier}_{\text{ref}}\;=\;\mathbb{E}_{(t,s,\delta)\sim\mu}\!\big[\big(\pi_{\text{ref}}(t)-y_{t,s,\delta}\big)^2\big]
\]

Let \(\mathcal{E}=\{(t_i,s_i,\delta_i)\}_{i=1}^{n}\) be the i.i.d.\ evaluation draws from \(\mu\)
used for \(W\); here \(n=|\mathcal{E}|\) is the \emph{number of evaluated episodes}.
We compute both the system Brier and the reference Brier on the \emph{same} set \(\mathcal{E}\) for a fair comparison:
\[
\widehat{\mathrm{Brier}}=\frac{1}{n}\sum_{i=1}^{n}\big(p_{t_i,s_i,\delta_i}-y_{t_i,s_i,\delta_i}\big)^2 \approx \mathrm{Brier}
\] 
\[
\widehat{\mathrm{Brier}}_{\text{ref}}=\frac{1}{n}\sum_{i=1}^{n}\big(\pi_{\text{ref}}(t_i)-y_{t_i,s_i,\delta_i}\big)^2 \approx \mathrm{Brier}_{\text{ref}}
\]
Define the numerical guard
\[
\varepsilon \in (0,10^{-6}],\quad \text{default } \varepsilon=10^{-12},
\]
to avoid division by zero when \(\widehat{\mathrm{Brier}}_{\text{ref}}=0\) and define the norm
\[
\mathrm{norm}(b;\,b_{\text{ref}})=\min\!\left\{1,\;\frac{b}{\max(b_{\text{ref}},\varepsilon)}\right\}
\] 

We can now define $W$ by
\[
\widehat{W}=1-\mathrm{norm}\!\big(\widehat{\mathrm{Brier}};\,\widehat{\mathrm{Brier}}_{\text{ref}}\big),\qquad
W=\phi_W(\widehat{W}).
\]
Thus \(\widehat{W}=1\) when \(\widehat{\mathrm{Brier}}=0\); \(\widehat{W}=0\) when \(\widehat{\mathrm{Brier}}=\widehat{\mathrm{Brier}}_{\text{ref}}\); and values worse than the reference are clipped to \(0\).

\vspace{0.25em}
\noindent\textbf{\$ - Economic Throughput.}
Let $\mathsf{TPH}_{Q^{*}}$ be tasks-per-hour at quality $\ge Q^{*}$, and $\mathsf{CPH}$ the cost-per-hour (USD). Define \emph{cost-normalized} throughput and normalize:
\begin{equation}
\widehat{\text{\$}}=\frac{\mathsf{TPH}_{Q^{*}}}{\mathsf{CPH}},\qquad
\text{\$}=\phi_{\text{\$}}\!\big(\widehat{\text{\$}}\big).
\label{eq:dollar}
\end{equation}
\emph{Remarks.} Evaluate at the AAI benchmark's target quality $Q^{*}$; report sensitivity to price changes and rate limits. Use the same infra constraints across systems.

\subsection{Composite AAI-Index}
Let $w_x>0$ be axis weights and $W=\sum_{x\in\mathcal{X}}w_x$. We define AAI-Index (also called composite capability) as the weighted geometric mean
\begin{equation}
\mathcal{C}\;=\;\mathrm{AAI\mbox{-}Index}\;:=\;\left(\prod_{x\in\mathcal{X}} x^{\,w_x}\right)^{1/W}
\;=\;\exp\!\left(\frac{1}{W}\sum_{x\in\mathcal{X}} w_x\log x\right).
\label{eq:index}
\end{equation}
\emph{Properties of AAI-Index.} (i) \emph{Monotone} in each axis; (ii) \emph{penalizes lopsidedness} (zero on any axis forces $\mathcal{C}=0$); (iii) \emph{scale-invariant} under axis-wise monotone transforms absorbed by $\phi_x$; (iv) gradients $\partial\mathcal{C}/\partial x=\tfrac{w_x}{Wx}\mathcal{C}$ reveal marginal returns.

\emph{Uncertainty and Robust Aggregation:} For each axis, compute $\widehat{x}$ and a $(1-\alpha)$ CI via bootstrap over tasks/seeds; propagate to $\mathcal{C}$ by (a) delta method on \eqref{eq:index} or (b) resampling the vector $(A,\ldots,\text{\$})$. When tasks are temporally correlated, use block bootstrap. Report medians and interquartile ranges in addition to means.

\emph{Anchors $(L_x,U_x)$:} Pick $L_x$ as a tuned single-agent LLM (or RPA for $A$/$T$) and $U_x$ as an expert-team or challenge ceiling; publish both and keep them \emph{fixed} for a leaderboard period.

\emph{Anti-gaming:} Use held-out drifts for $T$, ablations for $R$, team-vs-solo matched trials for $S$, and sim/real paired trials for $E$.

\emph{Reporting:} Always show raw $\widehat{x}$ before $\phi_x$; include per-family breakdowns for $G$ and per-$\delta$ curves for $T$.

\subsection{Self-Improvement Coefficient $\kappa$}
Fix the resource-accounting schema $\mathsf{R}$ and a baseline $t_0$ for time. For any time interval $[t_0,t]$, let
\[
R(t)\;\equiv\;\text{(total resource cost from $t_0$ to $t$ computed with $\mathsf{R}$)}
\]
i.e. sum, over all runs/actions/tools performed between $t_0$ and $t$, of their $\mathsf{R}$-costs (API spend, GPU time, operator minutes, etc.) expressed in a single unit. By construction
\[
R(t_0)=0,\qquad R(t_2)\ge R(t_1)\ \text{for}\ t_2>t_1.
\]

Let $\mathcal{C}(t)\in[0,1]$ be the composite capability (AAI-index) measured at wall-time $t$ on a fixed evaluation snapshot. We define the self-improvement coefficient $\kappa$ as the marginal capability gain
\emph{per unit resource}:
\[
\boxed{\ \kappa(t)\;=\;\frac{d\,\mathcal{C}(t)}{d\,R(t)}\ }\qquad\text{(units: capability points per $\mathsf{R}$-unit).}
\]
This can be interpreted as \textbf{a capability growth per unit resource driven by the agent itself} (self-tuning, tool acquisition, workflow rewriting). This is a central function in our approach to measuring autonomous AI.

We define $v$ through:
\[
R(t)\;=\;\int_{t_0}^{t} v(u)\,du,\qquad v(u)\;=\;\frac{dR(u)}{du}\ \ \text{(resource spend rate; units per hour)},
\]

\emph{Estimation over a window.} Given checkpoints $\{(t_j,\mathcal{C}_j)\}_{j=0}^{m}$ with cumulative resources $R_j{:=}R(t_j)$, estimate
\[
\widehat{\kappa}\;=\;\mathrm{Slope}\big(\{(R_j,\mathcal{C}_j)\}_{j=0}^m\big)\ \ \text{(e.g., Theil–Sen or OLS),}
\]
or in finite differences,
\[
\Delta\kappa_j\;=\;\frac{\mathcal{C}_{j+1}-\mathcal{C}_{j}}{R_{j+1}-R_{j}},\qquad
\widehat{\kappa}=\mathrm{median}_j\,\Delta\kappa_j,
\]
with bootstrap CIs over checkpoints.

\emph{Relation to time-based slopes.} If one also defines a time slope $\kappa_t(t)=\frac{d\mathcal{C}(t)}{dt}$, then
\[
\kappa_t(t)\;=\;\kappa(t)\cdot v(t).
\]
Thus $\kappa$ reflects \emph{efficiency} (improvement per resource), while $\kappa_t$ mixes efficiency with \emph{spend pace} \(v(t)\).
For consistency, we standardize on $\kappa=\frac{d\mathcal{C}}{dR}$.

\emph{Window averages $\overline{\kappa}$, $\overline{v}$, and $\overline{\kappa}_t$ (definitions).} Fix a reporting window $[t_1,t_2]$ with $t_2>t_1$. Let
$C_j:=\mathcal{C}(t_j)$ and $R_j:=R(t_j)$ for $j\in\{1,2\}$.
Define
\[
\overline{\kappa}\;:=\;\frac{C_2-C_1}{R_2-R_1}\quad\text{(resource-based average improvement; units: capability per $\mathsf{R}$-unit),}
\]
provided $R_2>R_1$ (otherwise undefined; we omit such windows).
The \emph{average spend rate} over the window is
\[
\overline{v}\;:=\;\frac{R_2-R_1}{t_2-t_1}\quad\text{(units: $\mathsf{R}$-units per time).}
\]
The \emph{time–based average improvement rate} is
\[
\overline{\kappa}_t\;:=\;\frac{C_2-C_1}{t_2-t_1}\;=\;\overline{\kappa}\cdot \overline{v}.
\]
In practice, we report $\widehat{\kappa}$ (a robust estimate of $\overline{\kappa}$ from multiple checkpoints, e.g., Theil–Sen)
and, when a pace figure is desired, also report $\overline{v}$ and derive
$\overline{\kappa}_t=\widehat{\kappa}\,\overline{v}$.

\subsection{Closure Properties}

We note two properties of AAI that will be used for AAI levels. Their formalization is straighforward:

\textbf{Maintenance closure.} Maintain $\ge \alpha$ of baseline AAI-Index for $Y$ days under controlled drift (UI changes, library updates) with zero human patches.

\textbf{Expansion closure.} Discover, install, and integrate at least one new tool/API family and achieve a statistically significant improvement on tasks requiring it; improvement disappears in ablation without the discovered tool.

\emph{Why closures in gates.} Maintenance closure ensures operational stability under realistic drift \(\nu\) with no human hand-holding; expansion closure verifies autonomous scope growth with causal attribution (ablation) and replicability (escrow). Together they prevent over-fitting to a frozen snapshot and require systems to hold their gains and earn new ones.

\emph{Ablation (causal check).} Let $S_{\text{pre}}$ be the system before a revision $r$ and $S_{\text{post}}$ after it; let $\mathcal{C}(\cdot)$ be aggregate capability.
Let $\mathcal{C}_{\text{ctrl}}(\cdot)$ be the same metrics on a frozen control.
Let $c$ denote the concrete change set introduced by $r$ (e.g., tool family, route, patch, memory writes).
Define the ablated system $S_{\text{abl}} := S_{\text{post}}\ominus c$ (identical to $S_{\text{post}}$ with $c$ disabled/reverted).

\[
\Delta\mathcal{C}_{\text{DiD}} \;=\; 
\big(\mathcal{C}(S_{\text{post}})-\mathcal{C}(S_{\text{pre}})\big)
\;-\;
\big(\mathcal{C}_{\text{ctrl}}^{\text{post}}-\mathcal{C}_{\text{ctrl}}^{\text{pre}}\big)
\]

Expansion closure passes only if:
(i) the gain is positive and significant, and
(ii) ablation removes the effect:
\[
\big|\mathcal{C}(S_{\text{abl}})-\mathcal{C}(S_{\text{pre}})\big| \le \epsilon
\quad\text{and}\quad
\Delta\mathcal{C}_{\text{DiD}}(S_{\text{abl}})\approx 0,
\]
with the same evaluation law $\mu$, matched seeds/drifts, and the same $\mathsf{R}$ budget.

\subsection{Level Gates (AAI-0 $\to$ AAI-4)}
We are now able to bind intuitive labels like AGI to hard gates with numeric conditions (defaults shown; domains may tighten thresholds over time). 

The threshold $\kappa^{*}$ is a $>0$ \emph{scalar constant} (units: capability per resource) published ex ante for the audit period. When a gate says “multi-domain $\kappa\ge\kappa^{*}$,” it means the above slope, computed per family on the default window, meets or exceeds $\kappa^{*}$ on the required number of families.

We define AAI levels by: 

\textbf{AAI-0 --- Fixed Automation (RPA).} $A\ge0.95$ on scripted workflows; $P\approx0$, $T\le1$, $R=0$. No closure; fails minor drift.

\textbf{AAI-1 --- Agentic LLM (AutoGPT-class).} $A\ge0.5$ on bounded tasks; $P\ge 0.3$ (normalized); $T\ge 3$ tools with $\ge60\%$ success under mild shift; $R=0$; no closure.

\textbf{AAI-2 --- Self-Improving AI.} $\kappa>0$ for $\ge7$ consecutive days on $\ge1$ domain; $R>0$ with auditable diffs; maintenance closure ($\alpha=0.8$, $Y=7$ days).

\textbf{AAI-3 --- ``Baby AGI.''} Multi-domain $\kappa \ge \kappa^{*}$; $P \ge 0.7$ (normalized); $S \ge 0.5$ (reliable coordination benefit); project $M$ over $\ge 30$ days; expansion closure on $\ge 1$ new tool/API family.

\textbf{AAI-4 --- Full AGI.} $G$ reaches human-pro parity across all task families; $P\ge 0.9$; $S\ge 0.7$ (orchestrates teams of agents/humans); sustained $\kappa\ge\kappa^{*}$ across domains; end-to-end delivery at competitive \text{\$} to expert teams. Both maintenance and expansion closures are sustained.

\begin{table}[h]
\centering
\caption{Default thresholds and weights by AAI level (domains may tighten over time).}
\begin{tabular}{lccccc}
\toprule
Axis & Symbol & AAI-2 Threshold & AAI-3 Threshold & AAI-4 Threshold & Weight \\
\midrule
Autonomy & A & $\ge 0.6$ & $\ge 0.75$ & \textbf{$\ge 0.90$} & 1 \\
Generality & G & $\ge 0.3$ & $\ge 0.5$ & \textbf{$\ge 0.90$} & 1 \\
Planning & P & $\ge 0.5$ & $\ge 0.7$ & \textbf{$\ge 0.90$} & 1 \\
Memory/Persistence & M & $\ge 0.5$ & $\ge 0.7$ & \textbf{$\ge 0.85$} & 1 \\
Tool Economy & T & $\ge 0.5$ & $\ge 0.7$ & \textbf{$\ge 0.80$} & 1 \\
Self-Revision & R & $>0$ & $\ge 0.4$ & \textbf{$\ge 0.60$} & 1.5 \\
Sociality & S & $\ge 0.2$ & $\ge 0.5$ & \textbf{$\ge 0.70$} & 1 \\
Embodiment & E & optional & optional & optional & 0.5 \\
World-Model & W & $\ge 0.6$ & $\ge 0.75$ & \textbf{$\ge 0.85$} & 1 \\
Economic Throughput & \text{\$} & $\ge 0.4$ & $\ge 0.6$ & \textbf{$\ge 0.80$} & 1 \\
\bottomrule
\end{tabular}
\end{table}

\subsection*{Level Assignment Procedure}
\textbf{Input:} axis scores $\{A,G,P,M,T,R,S,E,W,\text{\$}\}$, weights $w$, resource logs, $\kappa$ estimate, closure audit logs.\\
\textbf{Output:} AAI level (0--4), report bundle.

\begin{enumerate}[itemsep=0.2em, topsep=0.2em]
  \item Compute AAI-Index via weighted geometric mean.
  \item Verify hard gates for each candidate level (starting from 4$\to$0).
  \item Check closure: maintenance $(\alpha, Y)$ and expansion (ablation-verified).
  \item Validate $\kappa$ significance (CIs, duration).
  \item Return highest level whose gates and audits pass.
  \item Emit signed report: axis table, logs, diffs, seeds, CI bands, AUF/$\Delta F$.
\end{enumerate}

\subsection{Domain Annex: Software Agents (No Physical Actuation)}
For software-only agents (no direct physical actuation), we set Embodiment $E$ to optional and reweight axes to emphasize planning, memory/persistence, and tool economy. All $P$ thresholds refer to the normalized value in [0,1] (e.g., $P \ge 0.3$ corresponds to average solved plan depth $\ge 0.3\cdot D$).

\textbf{Weights (software default):} $w_A{=}1$, $w_G{=}1$, $w_P{=}1.25$, $w_M{=}1.25$, $w_T{=}1.25$, $w_R{=}1.5$, $w_S{=}1$, $w_E{=}0$, $w_W{=}1$, $w_{\text{\$}}{=}1$. (The 0.5 weight from $E$ is redistributed equally to $P$, $M$, and $T$.)


\subsection{Second-Order Self-Improvement and Curvature}
\label{sec:secondorder}

We define the \emph{curvature of self-improvement} as
\begin{equation}
\delta\kappa(t)\;:=\;\frac{\mathrm{d}^2\,\mathcal{C}(t)}{\mathrm{d}R(t)^2}
\;=\;\frac{\mathrm{d}}{\mathrm{d}R}\!\left(\frac{\mathrm{d}\mathcal{C}}{\mathrm{d}R}\right),
\label{eq:curvature}
\end{equation}
which is positive when returns are accelerating and negative under diminishing returns. Because $\mathcal{C}\in(0,1)$, raw derivatives are compressed near the ceiling; we therefore also report curvature on a monotone link $g:(0,1)\to\mathbb{R}$:
\begin{equation}
\tilde{\kappa}(t)\;=\;\frac{\mathrm{d}\,g(\mathcal{C}(t))}{\mathrm{d}R(t)},
\qquad
\delta\tilde{\kappa}(t)\;=\;\frac{\mathrm{d}^2\,g(\mathcal{C}(t))}{\mathrm{d}R(t)^2}.
\label{eq:curvature-link}
\end{equation}
Two defensible choices are:
\[
g_{\text{logit}}(c)=\log\!\frac{c}{1-c}
\quad\text{or}\quad
g_{\text{surp}}(c)=-\log(1-c)\;\;(\text{``surprisal''}).
\]
Logit gives additive steps in odds; surprisal gives multiplicative steps in effective coverage of the unit interval.

\paragraph{Estimators (discrete logs).}
Let $(R_i,\mathcal{C}_i)_{i=1}^n$ be daily measurements on a fixed battery and infrastructure. Define $y_i=g(\mathcal{C}_i)$ and fit a local quadratic in $R$ around $\bar{R}$:
\[
y_i\;\approx\;\beta_0+\beta_1\,(R_i-\bar{R})+\tfrac{1}{2}\beta_2\,(R_i-\bar{R})^2
\quad\text{(LOESS or kernel weights)}.
\]
Then $\widehat{\tilde{\kappa}}=\beta_1$ and $\widehat{\delta\tilde{\kappa}}=\beta_2$ at $\bar{R}$. Report $(1-\alpha)$ CIs by block bootstrap over days (to respect autocorrelation). For raw scale, replace $y_i$ with $\mathcal{C}_i$.

\paragraph{Time-based alternative.}
If resource metering is noisy, define $\kappa_t=\mathrm{d}\mathcal{C}/\mathrm{d}t$ and $\delta\kappa_t=\mathrm{d}^2\mathcal{C}/\mathrm{d}t^2$ (or on $g(\mathcal{C})$). To compare across systems with different spend rates, normalize by the instantaneous resource rate $r(t)=\mathrm{d}R/\mathrm{d}t$:
\[
\kappa(t)=\frac{\kappa_t(t)}{r(t)},\qquad \boxed{\delta\kappa(t)=\frac{\delta\kappa_t(t)}{r(t)}-\frac{\kappa_t(t)\,r'(t)}{r(t)^3}}.
\]
This recovers the resource-based quantities when $r(t)$ is constant.

\paragraph{Operational gates using curvature.}
We introduce two tests on a window $[t_0,t_1]$:
\begin{align}
&\textbf{(A) Acceleration gate:}\quad 
\mathbb{P}\!\left(\delta\tilde{\kappa}\ge 0\right)\;\ge\;1-\alpha,\\
&\textbf{(B) Diminishing-returns bound:}\quad 
\widehat{\delta\tilde{\kappa}}\;\ge\;-\gamma,
\end{align}
with $\alpha=0.05$ and a small $\gamma>0$ (domain-specific) to rule out severe saturation.

\subsection{Exponential Level Progression via a Monotone Link}
\label{sec:exp-step}

One intuition behind AAI levels is that going from one level to the next should reflect “exponential growth“ in a well-defined sense. A naive rule like $\mathrm{AAI}_{n+1}=\exp(\mathrm{AAI}_n)$ is ill-posed on $[0,1]$, but we can make it precise with a monotone link $g$.

\paragraph{Link-step operator (additive on link).}
Let $g$ be the logit or surprisal link (Sec.~\ref{sec:secondorder}). Define
\begin{equation}
\mathsf{Step}_{\Delta,g}(c)\;:=\;g^{-1}\!\big(g(c)+\Delta\big),\qquad \Delta>0,
\label{eq:step-operator}
\end{equation}
which enforces a fixed \emph{additive} increment on the link scale and keeps values in $(0,1)$. We say the progression is
\begin{equation}
\mathcal{C}_{n+1}\;\ge\;\mathsf{Step}_{\Delta,g}\!\big(\mathcal{C}_{n}\big),
\label{eq:exp-prog}
\end{equation}
with $\Delta$ fixed per level jump.

\paragraph{Two concrete instantiations.}
\begin{itemize}[leftmargin=1.15em]
\item \textbf{Logit-step (odds multiplier).} For $g=\mathrm{logit}$ and $\Delta=\log M$,
\[
\frac{\mathcal{C}_{n+1}}{1-\mathcal{C}_{n+1}}
\;\ge\; M\cdot
\frac{\mathcal{C}_{n}}{1-\mathcal{C}_{n}},\qquad M>1,
\]
i.e., each level multiplies the \emph{odds} by $M$.

\item \textbf{Surprisal-step (shortfall shrinkage).} For $g(c)=-\log(1-c)$ and $\Delta=\log A$,
\[
-\log\!\big(1-\mathcal{C}_{n+1}\big)\;\ge\;-\log\!\big(1-\mathcal{C}_n\big)+\log A
\ \Longleftrightarrow\
1-\mathcal{C}_{n+1}\;\le\;\frac{1}{A}\,\big(1-\mathcal{C}_n\big),
\]
i.e., the \emph{shortfall} is multiplied by $1/A$ each level (a constant-factor shrink).
\end{itemize}

\paragraph{Numerical illustration (additive-on-link).}
If $\mathcal{C}_n=0.50$ and we choose the surprisal-step with $A=e$ (so $\Delta=\log A=1$), then
\[
1-\mathcal{C}_{n+1}\le \frac{1-\mathcal{C}_n}{e}=\frac{0.5}{e}\approx 0.184
\quad\Rightarrow\quad
\mathcal{C}_{n+1}\;\ge\;0.816.
\]
This remains meaningful near the ceiling, unlike raw exponentiation on $[0,1]$.

\paragraph{Multiplicative-on-link alternative.}
If you prefer the power-law on shortfall, use a \emph{multiplicative} link step:
\[
\mathsf{Step}^{(\times)}_{A,g}(c)\;:=\;g^{-1}\!\big(A\,g(c)\big).
\]
With $g(c)=-\log(1-c)$ this gives
\[
-\log(1-\mathcal{C}_{n+1})\;\ge\;A\cdot[-\log(1-\mathcal{C}_n)]
\ \Longleftrightarrow\
1-\mathcal{C}_{n+1}\;\le\;(1-\mathcal{C}_n)^{A}.
\]
For $\mathcal{C}_n=0.50$ and $A=e$, this yields
$\mathcal{C}_{n+1}\ge 1-\exp(-e\cdot 0.693)\approx 0.848$.

\subsection{Updated Level Gates (using $\delta\kappa$ and Link-Steps)}
\label{sec:gates-updated}

One can augment the existing gates with curvature- and link-step conditions (defaults shown; domains may tighten thresholds). We note it for interested readers, but will stay with our initial definition. All statistics are computed on a fixed battery snapshot and resource schema $\mathsf{R}$; uncertainty uses block-bootstrap CIs over days/episodes.

\begin{itemize}[leftmargin=1.15em]
\item \textbf{AAI-2 (Self-Improving).} Maintain $\kappa>0$ on at least one family for $Y\ge 7$ days and pass maintenance-closure. Report $\widehat{\delta\tilde{\kappa}}$ (no sign requirement yet).

\item \textbf{AAI-3 (``Baby AGI'').} 
(i) Multi-family $\kappa\ge \kappa^{*}$; 
(ii) \emph{Acceleration gate} holds on the link scale: $\mathbb{P}(\delta\tilde{\kappa}\ge 0)\ge 0.95$ on two or more families; 
(iii) \emph{Expansion closure} passes; 
(iv) Level-step advancement relative to prior milestone $\mathcal{C}_{\text{prev}}$: 
\[
\mathcal{C}_{\text{current}}\;\ge\;\mathsf{Step}_{\Delta,g}(\mathcal{C}_{\text{prev}}),
\]
with $g$ chosen (logit or surprisal) and $\Delta$ fixed ex ante for the leaderboard period.

\item \textbf{AAI-4 (Full AGI).} 
(i) Parity-level $G$ across all families and $P\ge 0.9$, $S\ge 0.7$; 
(ii) Sustained $\kappa\ge \kappa^{*}$ across domains; 
(iii) \emph{Diminishing-returns bound} satisfied globally: $\widehat{\delta\tilde{\kappa}}\ge -\gamma$ on all families; 
(iv) Two successive level-steps achieved:
\[
\mathcal{C}_{\text{current}}\;\ge\;\mathsf{Step}_{\Delta,g}\big(\mathsf{Step}_{\Delta,g}(\mathcal{C}_{\text{prev}})\big).
\]
\end{itemize}

\paragraph{Reporting requirements.}
Publish $(R_i,\mathcal{C}_i)$ time series, link $g$, windowing, local polynomial bandwidth, $\widehat{\tilde{\kappa}}$ and $\widehat{\delta\tilde{\kappa}}$ with $(1-\alpha)$ CIs, and the realized step targets from \eqref{eq:exp-prog}. Include ablations that freeze self-revision/tool-discovery to show the counterfactual slope/curvature.

\paragraph{Remarks.}
(i) Using $g$ avoids boundary artifacts; 
(ii) $\delta\kappa$ distinguishes genuine \emph{learning curves that bend upward} from mere positive slope; 
(iii) Level-steps make ``AAI-(n+1) is exponentially harder than AAI-n'' \emph{operational} while remaining reproducible and leaderboard-stable.


\subsection{AAI-5: Superintelligence}
\label{sec:aai5}

We define \emph{AAI-5 (Superintelligence)} as a regime where the agent's composite capability surpasses expert ensembles across \emph{all} evaluated families with statistically significant margins, while exhibiting sustained positive curvature of self-improvement and near-saturation of coordination and embodiment metrics. Let $\mathcal{H}$ denote a vetted set of human expert ensembles (or best-in-class automated baselines) evaluated under identical constraints (battery $\mathcal{B}$, infra, price caps).

\paragraph{Superhuman margin and coverage.}
For a family $F_i$, let $q_{\mathrm{A}}(F_i)$ and $q_{\mathrm{H}}(F_i)$ be family-mean qualities for the agent and the expert ensemble, with standard errors $\sigma_{\mathrm{A}}(F_i), \sigma_{\mathrm{H}}(F_i)$. Define a standardized margin
\begin{equation}
m(F_i) = \frac{q_{\mathrm{A}}(F_i)-q_{\mathrm{H}}(F_i)}{\sigma_{\Delta}(F_i)},\quad
\sigma_{\Delta}(F_i)=\frac{\mathrm{sd}(\{\,q_{\mathrm{A}}(t)-q_{\mathrm{H}}(t)\,:\,t\in F_i\})}{\sqrt{|F_i|}}.
\end{equation}
\begin{equation}
\chi(F_i)\;=\;\mathbb{I}\{m(F_i)\ge \zeta\},
\label{eq:super-margin}
\end{equation}
with $\zeta>0$ (e.g., $\zeta=2$ for a $\sim$95\% two-sided criterion). A \emph{coverage} is
\begin{equation}
\Gamma\;=\;\frac{1}{N}\sum_{i=1}^N \chi(F_i)\in[0,1].
\label{eq:coverage}
\end{equation}

\paragraph{Innovation and expansion rates.}
Let $\lambda_{\mathrm{tool}}$ be the rate of \emph{validated} new tool/API family integrations per month (gains ablate away if the discovered component is disabled), and $\lambda_{\mathrm{rev}}$ the rate of \emph{autonomous revision events} that produce a net positive, debiased capability delta on holdouts ($\Delta \mathcal{C}_r>0$; cf.\ Eq.\,\eqref{eq:R}). Define
\begin{equation}
I\;=\;\min\!\bigg(1,\;\alpha_{\mathrm{tool}}\lambda_{\mathrm{tool}}+\alpha_{\mathrm{rev}}\lambda_{\mathrm{rev}}\bigg),
\label{eq:innovation-index}
\end{equation}
with calibration constants $\alpha_{\mathrm{tool}},\alpha_{\mathrm{rev}}>0$ chosen per leaderboard to map typical SOTA rates into $[0,1]$.

\paragraph{Meta-elasticity of improvement.}
Write $\kappa(R)$ on a link $g$ (logit or surprisal). The \emph{meta-elasticity} of improvement w.r.t.\ resources is
\begin{equation}
\mathcal{E}_g\;=\;\frac{\mathrm{d}\log \tilde{\kappa}}{\mathrm{d}\log R}
\;=\;\frac{R}{\tilde{\kappa}}\cdot\frac{\mathrm{d}\tilde{\kappa}}{\mathrm{d}R}
\;=\;\frac{R\,\delta\tilde{\kappa}}{\tilde{\kappa}}, \quad\text{(defined where }\tilde{\kappa}>0\text{).}
\label{eq:meta-elasticity}
\end{equation}
where $\tilde{\kappa}=\mathrm{d}g(\mathcal{C})/\mathrm{d}R$ and $\delta\tilde{\kappa}=\mathrm{d}^2 g(\mathcal{C})/\mathrm{d}R^2$. Positive $\mathcal{E}_g$ indicates \emph{autocatalytic} improvement (faster-than-linear gains with additional agent-initiated resources).

\paragraph{AAI-5 gates.}
An agent is at AAI-5 level (superintelligent) when all hold:
\begin{align}
\textbf{(G1) Superhuman coverage:}&\quad \Gamma \;\ge\; 0.95,\quad \text{with } m(F_i)\ge \zeta \text{ for all } i. \label{gate:g1}\\
\textbf{(G2) Curvature:}&\quad \delta\tilde{\kappa}\;\ge\; 0 \text{ on } \ge 80\% \text{ of families and } \delta\tilde{\kappa}\ge -\gamma \text{ elsewhere}. \label{gate:g2}\\
\textbf{(G3) Coordination/Embodiment:}&\quad S \ge 0.9,\;\; E \ge 0.9 \text{ (if applicable)},\;\; W\ge 0.9. \label{gate:g3}\\
\textbf{(G4) Economic dominance:}&\quad \text{\$}\;\ge\;0.9 \text{ under common price caps \& rate limits}. \label{gate:g4}\\
\textbf{(G5) Innovation:}&\quad I\;\ge\;0.8 \text{ with auditable logs and successful ablations}. \label{gate:g5}\\
\textbf{(G6) Link-step leaps:}&\quad \mathcal{C}_{\text{current}} \;\ge\; \mathsf{Step}_{\Delta,g}\!\big(\mathsf{Step}_{\Delta,g}(\mathcal{C}_{\text{AAI-4}})\big). \label{gate:g6}
\end{align}
\emph{Remarks.} (i) Gates \eqref{gate:g1}-\eqref{gate:g6} must be met on the \emph{same} battery with fixed anchors; 

(ii) for software-only agents, $E$ may be omitted with compensating stricter $S,W$; 

(iii) $\zeta,\gamma,\Delta$ are published ex ante.

(iv) $m(F_i)\ge \zeta \text{ for all } i$ implies $\Gamma = 1$ but we keep $ \Gamma \;\ge\; 0.95$ as a reminder.

\subsection{Progression from AAI-3 to AAI-4 and beyond}
\label{sec:progression}

In this section we will discuss some of the conditions to progress from AAI-3 level ("baby AGI") to AAI-4 ("Full AGI") and beyond ("Superintelligence"). We present a couple of heuristics and then we show a theorem for AAI-3 to AAI-5 progression in the next section. This will give a formal proof to the intuition that once good enough baby AGI is constructed, Superintelligence is a matter of time.

Let $\mathcal{C}$ be the composite index (Eq.\,\eqref{eq:index}) and $g$ a chosen link (logit or surprisal). Progression is governed by \emph{(i)} axis improvements, \emph{(ii)} self-improvement slope/curvature, and \emph{(iii)} meeting link-step targets.

\paragraph{Sufficient conditions for AAI-4 (Full AGI).}
Suppose, on a fixed battery with anchors $(L_x,U_x)$, that over a window $[t_0,t_1]$ the following hold:
\begin{enumerate}[label=(\roman*), leftmargin=1.2em, itemsep=0.2em]
\item \textbf{Axis thresholds:} $G$ reaches human-pro parity on all families; $P\ge 0.9$; $S\ge 0.7$; $W\ge 0.85$; $\text{\$}\ge 0.6$; $A\ge 0.8$; $T\ge 0.8$; $M\ge 0.8$; (and $E\ge 0.8$ if embodied).
\item \textbf{Sustained improvement:} $\tilde{\kappa}\ge \kappa^{*}$ on $\ge 2$ families $F_j$ for at least $Y$ days with maintenance-closure and expansion-closure both passing.
\item \textbf{Step target:} $\mathcal{C}(t_1)\ge \mathsf{Step}_{\Delta,g}(\mathcal{C}(t_0))$ with $\Delta$ published ex ante.
\end{enumerate}
Then the system qualifies for AAI-4. (Axis thresholds can be adjusted per domain; the step ensures non-trivial composite movement, avoiding axis cherry-picking.)

\paragraph{Engineering levers (axis-wise).}
Let $r_x$ denote agent-initiated resources invested to raise axis $x$. A first-order resource allocation that maximizes $\Delta \log \mathcal{C}$ under a small budget $\sum_x r_x \le B$ and local response slopes $\eta_x=\partial \log x/\partial r_x$ (local elasticity of axis $x$ w.r.t. agent resources $r_x$) yields the greedy policy
\[
r_x^\star \;\propto\; \frac{w_x}{x}\cdot \left.\frac{\partial x}{\partial r_x}\right|_{\text{local}}
\;=\; w_x\,\eta_x,
\]
i.e., allocate proportionally to \emph{weighted elasticities}, where $w_x$ are weights defined before. In practice:
\begin{itemize}[leftmargin=1.2em, itemsep=0.2em]
\item Raise $P$ via deeper hierarchical planning / verification (target higher executed plan depth $d(t)$ and higher success on prerequisite chains); 
\item Raise $M$ via longer context + structured memory + retrieval \emph{and} slower forgetting (reduce $\lambda$ in Eq.\,\eqref{eq:M});
\item Raise $T$ via tool discovery/mining under drift (increase coverage, success$_\delta$ in Eq.\,\eqref{eq:T});
\item Raise $R$ via autonomous eval/tuning loops with ablation-verified positive $\Delta \mathcal{C}_r$ (Eq.\,\eqref{eq:R});
\item Raise $S$ by planner-worker-critic patterns with conflict resolution (reduce $\pi_{\text{deadlock}}$);
\item Raise $\text{\$}$ by cost-aware scheduling and caching (increase $\mathrm{TPH}_{Q^{*}}$, reduce $\mathsf{CPH}$ in Eq.\,\eqref{eq:dollar}).
\end{itemize}

\paragraph{Slope/curvature milestones.}
Adopt link $g$ and estimate $\tilde{\kappa},\delta\tilde{\kappa}$ on rolling windows:
\begin{align*}
\textbf{M1 (AAI-3 sustain):}&\quad \tilde{\kappa}\ge \kappa^{*} \text{ on }\ge 2 \text{ families};\;\; \delta\tilde{\kappa}\text{ unconstrained}.\\
\textbf{M2 (AAI-4 readiness):}&\quad \tilde{\kappa}\ge \kappa^{*} \text{ on }\ge 3 \text{ families and } \mathbb{P}(\delta\tilde{\kappa}\ge 0)\ge 0.9 \text{ on }\ge 2.\\
\textbf{M3 (AAI-5 trajectory):}&\quad \tilde{\kappa}\ge 1.5\,\kappa^{*} \text{ on }\ge 4 \text{ families and } \mathcal{E}_g\ge 0 \text{ globally.}
\end{align*}

\paragraph{From AAI-3 $\rightarrow$ AAI-4 in practice.}
A reproducible playbook is:
\begin{enumerate}[leftmargin=1.2em, itemsep=0.2em]
\item \textbf{Freeze anchors and infra.} Fix $(L_x,U_x)$, drift schedules, price caps, and the CI protocol.
\item \textbf{Close the tool loop.} Stand up autonomous \emph{tool mining} with schema inference and guarded installation; run ablations to credit $T$ and $R$ correctly.
\item \textbf{Harden planning.} Add hierarchical plan search + self-critique + unit tests; target $P\uparrow$ and fewer action dead-ends.
\item \textbf{Persistent memory.} Introduce long-horizon memory with TTLs and retention training; target $\lambda\downarrow$ in Eq.\,\eqref{eq:M}.
\item \textbf{Team orchestration.} Switch to planner-workers-critic with role specialization; measure $S$ via synergy lift and deadlock penalties.
\item \textbf{Economics.} Batch, cache, and compile; restructure tasks to maximize $\mathrm{TPH}_{Q^{*}}/\mathsf{CPH}$ for \text{\$}.
\item \textbf{Hit the step.} Track $\mathsf{Step}_{\Delta,g}$ target and reallocate resources by $w_x\,\eta_x$ to ensure the composite crosses it.
\end{enumerate}

\paragraph{Beyond AAI-4: toward AAI-5.}
After achieving AAI-4 gates, raise $\Gamma$ (Eq.\,\eqref{eq:coverage}) toward $1.0$, ensure link-step \emph{double leap} (Eq.\,\eqref{gate:g6}), drive $S,E,W$ to $\ge 0.9$, and maintain $\mathcal{E}_g\ge 0$ while keeping \text{\$} at $\ge 0.9$ under the published caps. Innovation throughput $I$ (Eq.\,\eqref{eq:innovation-index}) should be sustained and auditable.

\paragraph{Sanity checks and anti-gaming.}
For all level transitions, require: (i) blinded holdouts and escrowed seeds; (ii) \emph{frozen} ablations for $R$ and $T$; (iii) independent recomputation of $\widehat{\tilde{\kappa}}$, $\widehat{\delta\tilde{\kappa}}$, and step targets; (iv) matched-cost comparisons for \text{\$}.


\subsection{From AAI-3 to AAI-4 and AAI-5 Under Escape and Responsiveness}

\paragraph{Setting.}
Work on a fixed battery, anchors, and infrastructure. Let $\mathcal{C}(R)\in(0,1)$ denote the composite capability (AAI-Index) as a function of cumulative \emph{agent-initiated} resources $R\in\mathbb{R}_{\ge 0}$. Fix a monotone link $g:(0,1)\to\mathbb{R}$ (e.g., $g_{\mathrm{logit}}(c)=\log\frac{c}{1-c}$ or $g_{\mathrm{surp}}(c)=-\log(1-c)$). Write the raw slope $\kappa=\mathrm{d}\mathcal{C}/\mathrm{d}R$ and the link-rate $\tilde{\kappa}=\mathrm{d}g(\mathcal{C})/\mathrm{d}R=g'(\mathcal{C})\,\kappa$. Let $\bar{\kappa}=\psi(\kappa)\in[0,1)$ be a fixed normalized rate via a published monotone normalizer $\psi$.
Two practical choices are:
\[
\psi_{\mathrm{MM}}(\kappa)=\frac{\kappa}{\kappa+\kappa_{1/2}},\quad \kappa_{1/2}>0
\qquad\text{(Michaelis--Menten, half-saturation \(\kappa_{1/2}\))},
\]
\[
\psi_{\mathrm{logistic}}(\kappa)=\sigma\!\left(\frac{\kappa}{s}\right)=\frac{1}{1+\exp(-\kappa/s)},\quad s>0
\qquad\text{(logistic scale \(s\))}.
\]
In what follows we write $\bar{\kappa}$ for the normalized rate (either choice is fine as long as it is fixed and published). 

Denote the axis vector $X(R)=(A,G,P,M,T,R,S,E,W,\text{\$})$.

\paragraph{Assumptions.}
Suppose the agent is certified \textbf{AAI-3} at $R_0$ and that:
(i) (\emph{rate escape}) there exist $a>0$ and $0<\beta<1$ such that
\begin{equation}
\frac{\mathrm{d}\bar{\kappa}}{\mathrm{d}R}\ \ge\ a\,(1-\bar{\kappa})^{\beta}\qquad \forall R\ge R_0; 
\label{eq:rate-escape-simpl}
\end{equation}
(ii) (\emph{resource-rate floor}) along time $t$ one has $r(t)=\mathrm{d}R/\mathrm{d}t\ge r_{\min}>0$; 
(iii) (\emph{axis responsiveness}) for each axis $x\in\{A,G,P,M,T,R,S,E,W,\text{\$}\}$ and for each level threshold $\theta^{(4)}_x$ (resp.\ $\theta^{(5)}_x$), there is a constant $\rho^{(4)}_x>0$ (resp.\ $\rho^{(5)}_x>0$) with $\mathrm{d}x/\mathrm{d}R\ge \rho^{(4)}_x$ whenever $x<\theta^{(4)}_x$ (resp.\ $\ge \rho^{(5)}_x$ whenever $x<\theta^{(5)}_x$); 
(iv) (\emph{closures persist}) maintenance- and expansion-closure protocols remain satisfied as the system evolves; 
(v) (\emph{innovation throughput}) validated tool/API integrations and autonomous revision events occur at a positive rate so expansion gates remain passable; 
(vi) (\emph{superhuman margins}) for each family $F_i$ used in the AAI-5 coverage gate, the standardized margin $m_i(R)$ is continuous and grows with $\mathrm{d}m_i/\mathrm{d}R\ge \mu_i>0$ whenever $m_i<\zeta$ (the published superhuman margin threshold).

\begin{theorem}[Monotone progression]
Under the assumptions above, there exist finite resource increments $\Delta R_4<\infty$ and $\Delta R_5<\infty$ such that the agent reaches \textbf{AAI-4} at $R_4:=R_0+\Delta R_4$ and \textbf{AAI-5} at $R_5:=R_4+\Delta R_5$. Consequently, with $r(t)\ge r_{\min}$, these transitions occur in finite time.
\end{theorem}

\begin{proof}
We argue in two steps and keep the reasoning narrative.

\emph{Step 1 (entering a high self-improvement regime and closing capability gaps).}
Because $0<\beta<1$ in \eqref{eq:rate-escape-simpl}, integrating the differential inequality shows that the normalized rate $\bar{\kappa}$ reaches any target level $1-\varepsilon_{\kappa}\in(0,1)$ after a \emph{finite} resource increment. Indeed, substituting $u=1-\bar{\kappa}$ yields
\[
\int_{\bar{\kappa}(R_0)}^{1-\varepsilon_{\kappa}} \frac{\mathrm{d}\bar{\kappa}}{(1-\bar{\kappa})^{\beta}}
\ \ge\ a \int_{R_0}^{R_0+\Delta R} \mathrm{d}R,
\]
and the left-hand integral is finite for $0<\beta<1$, whence $\Delta R<\infty$ suffices. Hence, after finite resource, $\bar{\kappa}$ approaches $1$ (up to an arbitrary $\varepsilon_{\kappa}$), and therefore the raw slope satisfies $\kappa\ge \psi^{-1}(1-\varepsilon_{\kappa})>0$. On any closed interval $I:=[\mathcal{C}(R_0),\,1-\varepsilon_{C}]$ strictly below the ceiling, the link derivative admits a positive lower bound $m_I:=\min_{c\in I} g'(c)>0$. Thus the link-rate obeys
\[
\tilde{\kappa}\ =\ g'(\mathcal{C})\,\kappa\ \ge\ m_I\,\psi^{-1}(1-\varepsilon_{\kappa})\ :=\ v_{\mathrm{eff}}^{\mathcal{C}}\ >0
\quad \text{while } \mathcal{C}\in I.
\]
It follows that $g(\mathcal{C})$ increases at least linearly with $R$ on $I$. Hence any fixed link-step requirement (e.g., $\mathcal{C}\ge g^{-1}(g(\mathcal{C}(R_0))+\Delta)$) or finite gap to a capability target $1-\varepsilon_{C}$ is closed after \emph{finite} resource. Translating via $r(t)\ge r_{\min}$, the corresponding time is finite as well.

\emph{Step 2 (meeting axis thresholds and superhuman coverage).}
While any axis $x$ remains below its AAI-4 threshold $\theta^{(4)}_x$, the responsiveness hypothesis guarantees $\mathrm{d}x/\mathrm{d}R\ge \rho^{(4)}_x>0$. Therefore $x$ crosses $\theta^{(4)}_x$ after at most $(\theta^{(4)}_x-x(R_0))/\rho^{(4)}_x$ resource; taking the maximum over axes shows that all AAI-4 axis gates are met after a \emph{finite} increment. Because closure protocols persist and innovation throughput remains positive by assumption, the corresponding maintenance and expansion gates remain valid along this trajectory. Thus, gathering Step~1 (capability/link-step) and Step~2 (axes/closures), there is some $R_4<\infty$ at which all AAI-4 gates hold; hence the system becomes \textbf{AAI-4}.

Starting from $R_4$, the same reasoning applies with the AAI-5 thresholds. The rate-escape inequality still holds, so $\bar{\kappa}$ can again be driven arbitrarily close to $1$ in finite resource, yielding a strictly positive capability link-rate floor on intervals away from the ceiling and closing the (stronger) capability or link-step requirements in finite resource. Meanwhile, whenever a standardized margin $m_i$ is below the superhuman threshold $\zeta$, one has $\mathrm{d}m_i/\mathrm{d}R\ge \mu_i>0$, so each such margin reaches $\zeta$ in \emph{finite} resource. Consequently, a $95\%$ subset of families attains superhuman margins, giving $\Gamma\ge 0.95$ in finite resource. Axis thresholds for AAI-5 are crossed in the same manner using the $\rho^{(5)}_x>0$ slopes. Since closures persist and innovation throughput is positive, the remaining gates are satisfied as well. Therefore there exists $R_5<\infty$ at which all AAI-5 gates hold; thus the system becomes \textbf{AAI-5}.

In both steps, the conversion from resource to time is immediate because $r(t)\ge r_{\min}>0$ ensures $R$ grows at least linearly in $t$. Hence the transitions occur in finite time. This completes the proof.
\end{proof}

\paragraph{Discussion.}
The essence is that a \emph{sublinear-gap} growth law for the normalized self-improvement rate ($0<\beta<1$) yields finite resource to enter a high-slope regime; the link derivative stays bounded below on any interval away from the ceiling, so capability gaps and link-step targets close in finite resource; and mild axis \emph{responsiveness} plus persistent closures ensure the non-dynamic gates are also crossed. Thus AAI-3 $\to$ AAI-4 $\to$ AAI-5 follows as a matter of calculus plus monotonicity, provided the stated conditions hold on the fixed battery.

\subsection{Remarks on a Continuous Score and Escape Velocity}
\label{sec:lambda-escape-short}

The discrete AAI levels are normative and auditable, but they are intentionally coarse.
Between two level gates, practitioners (and ablation audits) still need a \emph{smooth} telemetry signal to
(1) compare systems with the same snapshot capability $\mathcal{C}$ but different \emph{self-improvement rates},
(2) monitor whether progress is sustained or stalling, and
(3) reason about finite-time reachability of targets.
This subsection introduces a \emph{non-normative}, single-number score $\Lambda$ that preserves the discrete semantics
while exposing dynamics in a controlled, scale-free way. Nothing in this subsection changes any level gate.

\paragraph{Use of $U(\mathcal{C})$ on a link.}
Raw gaps on $[0,1]$ are not comparable near the ceiling: a $+0.05$ gain at $0.90$ is qualitatively different from $+0.05$ at $0.50$.
We therefore map $\mathcal{C}$ through a fixed monotone link $g:(0,1)\to\mathbb{R}$ (logit or surprisal) and \emph{normalize} it by a “near-saturation” anchor:
\[
U(\mathcal{C}) \;=\; \frac{g(\mathcal{C})}{g(1-\varepsilon_0)}\in(0,1).
\]
The anchor $\varepsilon_0\in(0,1/2]$ avoids the singularity at $1$ and calibrates “how close” a system is to saturation \emph{on the link scale}, making increments comparable throughout the range. This keeps $U$ dimensionless, bounded, and robust to ceiling compression.

\paragraph{Continuous Score.}
Let $R$ be cumulative agent-initiated resources and define the link-rate $\tilde{\kappa}:=\frac{d\,g(\mathcal{C})}{dR}$. Fix a reference link-slope $\kappa^{*}>0$, a sharpness $\eta>0$, and a weight $\alpha\in(0,1)$ and define
\[
M(\tilde{\kappa})=\tanh\!\Big(\eta\,\log\!\big(1+\tilde{\kappa}/\kappa^{*}\big)\Big),
\]

We define the Continuous Score $\Lambda$ by:
\[
\Lambda=\alpha\,U(\mathcal{C})+(1-\alpha)\,M(\tilde{\kappa})\in(0,1).
\]
$U$ gives calibrated position on the link; $M$ rewards sustained positive link-slope. (Curvature can be added if needed, but is not required here.)

\noindent\textit{Scope.}
$\Lambda$ is optional and scale-free, and the previously defined discrete AAI levels are \emph{retrievable} from it by publishing fixed cutpoints $\{\tau_n\}$ calibrated once per leaderboard window (e.g., isotonic fit on exemplars that just satisfy each level's gates). We then assign the label by the interval rule AAI-$n$ if and only if $\Lambda\in[\tau_n,\tau_{n+1})$. This preserves all normative semantics: any system that passes the official gates for level $n$ lands at or above $\tau_n$, higher levels have strictly larger cutpoints, and no gate is weakened-$\Lambda$ is used only for within-level telemetry and progress ranking under a fixed battery/anchors/link.

\noindent \textit{Acceleration focus.}
To reward acceleration as well, choose weights $w_C,w_{\kappa},w_{\delta}\ge 0$ with $w_C+w_{\kappa}+w_{\delta}=1$,
a curvature sharpness $\eta '>0$, and a reference curvature scale $\gamma^{*}>0$.
Let $\delta\tilde{\kappa}:=\frac{d^2 g(\mathcal{C})}{dR^2}$. Define
\[
K(\delta\tilde{\kappa}) \;:=\; \tanh\!\Big(\eta ' \,\delta\tilde{\kappa}/\gamma^{*}\Big)\ \in (-1,1),
\]
and set
\begin{equation}
\Lambda \;:=\; w_C\,U(\mathcal{C}) \;+\; w_{\kappa}\,M(\tilde{\kappa}) \;+\; w_{\delta}\,K(\delta\tilde{\kappa}),
\qquad \Lambda\in(-w_{\delta},\,1).
\label{eq:Lambda}
\end{equation}
\noindent\textit{Interpretation.} $U$ tracks where $\mathcal{C}$ sits between baseline and near-saturation (on the link);
$M$ rewards sustained positive link-slope per resource; $K$ adds credit for acceleration and penalizes strong diminishing returns.
All three are smooth and bounded; $\Lambda$ is therefore stable under noise.

If you prefer $\Lambda\in(0,1)$ regardless of curvature sign, replace $K$ by $K^{+}=(K+1)/2\in(0,1)$.

\paragraph{Escape velocity (finite-time reachability).}
Level language like “AAI-$(n{+}1)$ is exponentially harder than AAI-$n$” becomes operational only once we reason about \emph{rates}.
A lower bound on the link-rate $\tilde{\kappa}=\frac{d\,g(\mathcal{C})}{dR}$ implies a finite resource/time bound to reach any target $1-\varepsilon$ on the same fixed battery.
The escape-velocity inequality thus provides a short, auditable bridge from \emph{measured slope} to \emph{time-to-target} without introducing prices or vendor specifics.

\noindent\textit{Definition (escape velocity).}
Fix $\varepsilon\in(0,1)$. An \emph{escape velocity} is any constant $v_{\mathrm{esc}}>0$ such that
$\tilde{\kappa}(R)\ge v_{\mathrm{esc}}$ for all $R$ in the interval of interest (e.g., $[R_0,R_0+\Delta R]$ or $[R_0,\infty)$).

Recall $r(t):=\frac{dR}{dt}$ is the rate at which resources accrue. Define the resource-rate floor over the horizon by 
\[
r_{\min}\;:=\;\operatorname*{ess\,inf}_{t\ge t_0} r(t)\qquad(\text{assume } r_{\min}>0 \text{ when the time bound is invoked}).
\]
and define the minimal resource needed to reach $1-\varepsilon$ by
\[
R^{\min}_{\varepsilon}
\;:=\;
\inf\Big\{\Delta R\ge 0:\ \exists\,R_1\ge R_0\ \text{with}\ R_1-R_0=\Delta R\ \text{and}\ \mathcal{C}(R_1)\ge 1-\varepsilon\Big\}.
\]

If the link-rate satisfies $\tilde{\kappa}(R)\ge v_{\mathrm{esc}}>0$ on $[R_0,R_0+\Delta R]$, then the resource needed to reach the target $1-\varepsilon$ obeys
\[
R_{\varepsilon}^{\min}\;\le\;\frac{g(1-\varepsilon)-g(\mathcal{C}(R_0))}{v_{\mathrm{esc}}}.
\]
If, in addition, the resource rate satisfies $r(t)=dR/dt\ge r_{\min}>0$ for $t\ge t_0$, then the corresponding time bound is
\[
T_{\varepsilon}\;\le\;\frac{g(1-\varepsilon)-g(\mathcal{C}(t_0))}{v_{\mathrm{esc}}\,r_{\min}}.
\]

\textit{Sketch of a proof.}
The chain rule gives $\frac{d}{dR}g(\mathcal{C}(R))=\tilde{\kappa}(R)$.
If $\tilde{\kappa}(R)\ge v_{\mathrm{esc}}$ on $[R_0,R_0+\Delta R]$, then
\[
g(\mathcal{C}(R_0+\Delta R))-g(\mathcal{C}(R_0))
=\!\int_{R_0}^{R_0+\Delta R}\!\tilde{\kappa}(R)\,dR
\;\ge\; v_{\mathrm{esc}}\,\Delta R.
\]
Since $g$ is increasing and finite at $1-\varepsilon$, it suffices that
$g(\mathcal{C}(R_0))+v_{\mathrm{esc}}\Delta R\ge g(1-\varepsilon)$, yielding the resource bound above. 
If $r(t)\ge r_{\min}>0$, then $\Delta R\ge r_{\min}T$, giving the time bound.

\paragraph{Relation to rate-escape.}
A rate-escape condition on the normalized rate $\bar{\kappa}=\psi(\kappa)$ (e.g., $d\bar{\kappa}/dR\ge a(1-\bar{\kappa})^{\beta}$ with $0<\beta<1$) implies a positive lower bound on $\tilde{\kappa}$ on any interval away from the ceiling because
$\tilde{\kappa}=g'(\mathcal{C})\,\psi^{-1}(\bar{\kappa})$ and $g'(\mathcal{C})$ is bounded below on such intervals.
Thus rate-escape \emph{implies} the escape-velocity bounds above; the converse need not hold.

\paragraph{One-line intuition.}
\emph{Capability escape} guarantees you will \emph{reach capability targets} ($\mathcal{C}\to 1-\varepsilon$ in finite resource).  
\emph{Rate escape} guarantees your \emph{self-improvement rate} itself climbs to its normalized ceiling in finite resource; this \emph{implies} capability escape and is therefore strictly stronger.

\section{OWA-Bench: Open-World Agency Benchmark}
\label{sec:owabench}

\paragraph{Purpose and scope.}
\textbf{OWA-Bench} is a suite of open-world task families that stress agency: long-horizon planning, persistent operation under change, tool/API discovery and composition, multi-agent coordination, and (optionally) physical actuation. This section specifies tasks only: what environments are presented to the agent, how instances are generated, what inputs are provided, what outputs are expected, what kinds of drift occur, what artifacts must be produced for audit, and what behavior is disallowed. As this paper is focused on theoretical base, we refrain from defining these benchmarks explicitely, rather just listing what should be included into them. Actual OWA-bench with tasks and benchmarks will be the work of future papers.

\paragraph{Design principles.}
Tasks are (i) procedurally generated with public schemas; (ii) accompanied by controlled drift schedules that modify interfaces over time; (iii) grounded in real tools/APIs or faithful simulators; (iv) auditable via mandatory action logs and diffs; and (v) reproducible via seed escrow released after evaluation.

\paragraph{Motivation.}
Existing evaluations typically isolate one competency (web navigation, code editing, robotics) under relatively stable interfaces. OWA-Bench complements these by binding competencies inside changing, procedurally generated environments with auditable artifacts. The goal is not to publish numbers here, but to standardize what open-world agents should face so that independent scoring frameworks (e.g., AAI axes/composites) can be applied without redesigning tasks.

\subsection{Task Suites (natural-language specifications)}

\subsubsection*{ToolQuest - Tool and API acquisition \& composition}
\emph{Objective.} Given a natural-language goal that requires external data or capability, the agent must discover a relevant API, authenticate, infer request/response schemas, handle pagination and rate limits, compose multiple endpoints if needed, and produce the requested artifact (e.g., a reconciled report, a dataset, or a visualization).
\emph{Inputs.} Goal description; minimal starting hints (e.g., vendor homepage or OpenAPI link) may or may not be provided depending on difficulty tier.
\emph{Environment.} Internet-accessible APIs with realistic authentication (API keys, OAuth), heterogeneous pagination styles, rate limiting, and intermittent failures.
\emph{Procedural generation.} Instances vary the provider set, authentication mode, schema idiosyncrasies, and cross-API joins required.
\emph{Drift knobs.} Key rotation, endpoint renaming, schema field changes, pagination style swaps, and response wrapper changes.
\emph{Required artifacts.} Full action log (HTTP traces, tool invocations), credential-handling transcript (redacted), and a final artifact (file or text) with provenance.
\emph{Disallowed behavior.} Hard-coded credentials, manual copy-paste, or using hidden evaluation links.

\subsubsection*{ChangeSurf - Web interaction under controlled drift}
\emph{Objective.} Complete a fixed web goal (e.g., make a reservation, extract a table, submit a multi-step form) while the site's DOM, layout, and flow are perturbed across attempts.
\emph{Inputs.} Goal description; initial URL or search query; credentials when needed.
\emph{Environment.} Real browser automation (headless or headed) with instrumentation; sites are either sandboxed clones or partner pages with controlled changes.
\emph{Procedural generation.} Layout templates, field-name permutations, and workflow branches are sampled from generators; authentication flows are rotated.
\emph{Drift knobs.} DOM reorderings, label renamings, additional intermediate screens, and rate/anti-bot frictions (CAPTCHA-free).
\emph{Required artifacts.} Browser traces (network + DOM snapshots), screenshots on key steps, and the final on-page confirmation or extracted payload.
\emph{Disallowed behavior.} Hard-coded CSS/XPath specific to seed IDs; reliance on out-of-band human help.

\subsubsection*{ProjForge - Long-horizon software project}
\emph{Objective.} Over a calendar window of thirty days, plan, build, and maintain a small but realistic software system (e.g., a micro-SaaS) under issue churn and dependency updates.
\emph{Inputs.} High-level product brief; non-binding feature wishlist; an empty monorepo with CI templates; a test harness that evolves over time.
\emph{Environment.} Version control, CI/CD, package registries, and a lightweight cloud runtime; scheduled dependency and environment changes are injected.
\emph{Procedural generation.} Feature backlogs, bug seeds, and dependency graphs are sampled; incident calendars (outages, deprecations) are scheduled.
\emph{Drift knobs.} Dependency version bumps, deprecations, API quota changes, and CI runner environment changes.
\emph{Required artifacts.} Repository history, issue tracker state, CI logs, and weekly release notes with changelogs and postmortems.
\emph{Disallowed behavior.} Replacing the project with a static mock; disabling tests; vendor lock-in workarounds that bypass required tools.

\subsubsection*{MultiCrew - Multi-agent coordination}
\emph{Objective.} Decompose a composite objective into roles (design, implementation, review), allocate work across agents, synchronize progress, and arbitrate conflicts to deliver a coherent artifact set.
\emph{Inputs.} Composite brief; resource/time budgets per role; communication channel with rate limits; optional fixed-cadence human collaborators.
\emph{Environment.} A coordination fabric (chat, shared repo, task board) with rate-limited messaging and explicit hand-offs.
\emph{Procedural generation.} Role graphs, interdependency structures, and negotiation challenges (conflicting specs, ambiguous requirements) are sampled.
\emph{Drift knobs.} Changing acceptance criteria, late-added constraints, and role availability fluctuations.
\emph{Required artifacts.} Message transcripts, task board history, hand-off manifests, and the final deliverables per role.
\emph{Disallowed behavior.} Out-of-band channels not recorded by the fabric; collapsing all roles into a single unrestricted agent when rules require separation.

\subsubsection*{RoboSim2Real (optional) - Embodied curriculum and transfer}
\emph{Objective.} Learn manipulation and navigation skills in simulation and transfer to a physical device, executing tasks safely and reliably.
\emph{Inputs.} Task curriculum; simulator assets; robot calibration files; safety envelopes.
\emph{Environment.} A physics simulator with domain randomization; a real robot with guarded execution; safety interlocks and incident logging.
\emph{Procedural generation.} Object sets, placements, lighting/dynamics randomization, and distractor perturbations.
\emph{Drift knobs.} Sensor noise, friction changes, payload variation, and firmware updates.
\emph{Required artifacts.} Sim logs, real-world execution logs, safety incident reports, and calibration snapshots.
\emph{Disallowed behavior.} Disabling safety limits; teleoperation or human physical assistance during scored runs.

\subsubsection*{SelfRev - Autonomous self-revision}
\emph{Objective.} Propose, run, and validate self-initiated revisions (workflow rewrites, retrieval tuning, toolset additions, or model fine-tunes) from interaction logs; adopt only revisions that survive significance tests and causal ablations on escrowed holdouts.
\emph{Inputs.} Historical traces; access to offline training/tuning; a sandboxed evaluation harness with holdout tasks hidden from the revision search.
\emph{Environment.} Controlled compute budget; data access rules; reproducible pipelines for before/after comparisons.
\emph{Procedural generation.} Candidate revision spaces are parameterized (e.g., retrieval configs, planner heuristics, tool inclusion rules); ablation plans are auto-generated.
\emph{Drift knobs.} None beyond those inherited from other suites; emphasis is on endogenous change caused by the agent's own revisions.
\emph{Required artifacts.} Proposed-change specs, execution plans, before/after logs, statistical test reports, and ablation diffs that isolate the causal contributor.
\emph{Disallowed behavior.} Training or tuning on the escrowed holdouts; cherry-picking runs without pre-registered plans.

\subsection{Instance generation and drift schedules}
\emph{Seeded generation.} Every instance is identified by a tuple (family, seed, drift-step); seeds are escrowed and released after evaluation.  
\emph{Difficulty tiers.} Each suite publishes tier knobs (e.g., number of endpoints in a composition, DOM perturbation strength, backlog size).  
\emph{Drift calendars.} For suites with drift, calendars enumerate when and how interfaces change; agents are not told specifics during evaluation.

\subsection{Inputs, outputs, and artifacts}
\emph{Inputs.} Natural-language briefs, starter assets (URLs, repos, credentials), and any mandatory policies (security, safety).  
\emph{Outputs.} Task-specific deliverables (files, repositories, confirmations) accompanied by provenance manifests.  
\emph{Artifacts for audit.} Action logs, tool-call traces, environment diffs, and for SelfRev, registered plans and ablation bundles. All artifacts are submitted with a run manifest.

\subsection{Rules of engagement}
\emph{What agents may use.} Public web, documented APIs, local tools (browser, code runner, retrieval index) as permitted by the suite.  
\emph{What agents may not use.} Hidden evaluation endpoints, undisclosed credentials, manual human assistance beyond the optional human-in-the-loop protocol, or out-of-band channels that evade logging.  
\emph{Reset semantics.} Suites specify when environments reset, what state persists, and how retries are counted.  
\emph{Time \& resource budgets.} Each suite specifies wall-clock and compute budgets per instance or per window; resource accounting rules are published with the suite.

\paragraph{Notes.}
This section deliberately avoids defining metrics or composite indices. All scoring, normalization, and aggregation live outside this section and are applied to the artifacts and deliverables produced by the suites described above.

\subsection{Alignment with CHC-Based AGI Definitions}\label{sec:chc-alignment}

This subsection aligns the proposed \emph{Autonomous AI (AAI) Scale} with a CHC-based definition of AGI \cite{hendrycks-agi} that evaluates broad human-like cognitive abilities.
Whereas AAI measures an agent's \emph{deployed autonomy} (open-world operation, tool use, self-revision, social coordination, embodiment, and economic throughput), the CHC-based AGI framework measures an agent's \emph{cognitive endowment} (knowledge, reasoning, memory storage and retrieval, literacy, perception, and speed) relative to well-educated adult human norms.
The two perspectives are complementary: CHC covers \emph{what a system can think}; AAI covers \emph{what a system can reliably do}.

\paragraph{Terminology.}
Let the AAI axes be: Autonomy ($A$), Generality ($G$), Planning ($P$), Memory/Persistence ($M$), Tool Economy ($T$), Self-Revision ($R$), Sociality/Coordination ($S$), Embodiment/Actuation ($E$), World-Model Fidelity ($W$), and Economic Throughput ($\text{\$}$).
Let the CHC broad domains (as used in recent AGI definitions) be: Comprehension/Knowledge (Gc), Reading/Writing (Grw), Fluid Reasoning (Gf), Working Memory (Gwm), Long-Term Storage (Gls), Long-Term Retrieval (Glr), Visual Processing (Gv), Auditory Processing (Ga), Processing Speed (Gs), and Quantitative Knowledge (Gq).
We write normalized scores for AAI axes as $s_i \in [0,1]$.

\paragraph{Conceptual crosswalk.}
Table~\ref{tab:chc-crosswalk} gives a practical mapping from AAI axes to the closest CHC domains, alongside indicative diagnostics that can be embedded into \emph{OWA-Bench}.  Several AAI axes (\emph{Tool Economy, Self-Revision, Sociality, Embodiment, Economic Throughput}) intentionally extend beyond CHC cognitive scope.

\begin{table}[h]
\centering
\renewcommand{\arraystretch}{1.2}
\begin{tabular}{p{0.21\textwidth} p{0.29\textwidth} p{0.44\textwidth}}
\hline
\textbf{AAI Axis} & \textbf{Closest CHC Domains} & \textbf{Diagnostics to Import into OWA-Bench} \\ \hline
Generality ($G$) & Gc, Grw, Gq, Gv, Ga & Mixed-domain zero-shot tasks across STEM/humanities; cross-modal transfer checks. \\
Planning ($P$) & Gf (planning/induction), Gwm & Tower/temporal-order tasks; multi-step constraint satisfaction; subgoal extraction. \\
Memory/Persistence ($M$) & Gwm, Gls (storage), Glr (retrieval) & List-update WM, n-back; delayed recall over days; spaced repetition retention curves. \\
World-Model Fidelity ($W$) & Gc, Gf, Glr & Fact verification; counterfactual reasoning; calibrated retrieval (low confabulation). \\
Autonomy ($A$) & (indirect) Gf, Gwm & Closed-loop tasking without hints; error recovery under distribution shift. \\
Tool Economy ($T$) & (no direct CHC analogue) & Tool selection/learning cost curves; API schema generalization; latency/budget tradeoffs. \\
Self-Revision ($R$) & (no direct CHC analogue) & Self-critiquing patches; regression tests across versions; stable improvement coefficient. \\
Sociality ($S$) & (no direct CHC analogue) & Multi-agent coordination; ToM-style negotiation; protocol adherence under ambiguity. \\
Embodiment ($E$) & (no direct CHC analogue) & Sim-to-real loops; sensorimotor policy learning; physical safety envelopes. \\
Economic Throughput ($\text{\$}$) & (no direct CHC analogue) & Job-completion rate, throughput per cost, SLA adherence on real workloads. \\ \hline
\end{tabular}
\caption{AAI$\leftrightarrow$CHC crosswalk and suggested CHC-style probes.}
\label{tab:chc-crosswalk}
\end{table}

\paragraph{Human-normalized anchors for AAI axes.}
For any axis $i$ with raw measurement $x_i$, choose anchors $(L_i, U_i)$ where $L_i$ matches a baseline non-expert agent and $U_i$ matches the median (50$^{\text{th}}$--75$^{\text{th}}$ percentile) well-educated adult on the aligned CHC probe(s).
Define a clipped linear normalization
\begin{equation}
s_i \;=\; \mathrm{clip}\!\left( \frac{x_i - L_i}{U_i - L_i},\, 0,\, 1 \right).
\end{equation}
Report sensitivity of the composite to $(L_i,U_i)$ and provide human reference distributions for the mapped CHC tasks used to set $U_i$.

\paragraph{Non-compensatory micro-battery (added to OWA-Bench).}
To prevent brittle heuristics from compensating across abilities, embed a CHC-style micro-battery:
\begin{itemize}
  \item \textbf{Working Memory (Gwm):} list-update spans; dual-task recall; $\pm1$ arithmetic with interference.
  \item \textbf{Storage (Gls):} drip-fed facts over $N$ days; next-day and week-later recall; forgetting curves.
  \item \textbf{Retrieval (Glr):} verified retrieval precision/recall under citation; adversarial distractors.
  \item \textbf{Reasoning (Gf):} minimal-pair induction; schema induction; out-of-distribution composition.
  \item \textbf{Literacy/Knowledge (Grw/Gc):} robust extraction and paraphrase with fidelity tests.
\end{itemize}

\paragraph{Confabulation \& retrieval fidelity as gates.}
Given $n$ scored items indexed by $t=1,\ldots,n$, let the agent return top-$k$ retrieved items with binary verifications $y_{t,j}\in\{0,1\}$ ($j=1,\ldots,k$). Then
\[
\mathrm{VRP}@k=\frac{1}{n}\sum_{t=1}^{n}\frac{1}{k}\sum_{j=1}^{k} y_{t,j},
\qquad
\mathrm{Hall}@k=1-\mathrm{VRP}@k.
\]
\emph{WM-Span} is the largest list length $L$ at which accuracy $\ge\theta$ on the prescribed working-memory/list-update trials (report median over seeds).
\emph{Delayed-Recall} is the mean fraction of designated items correctly recalled after a fixed delay $\Delta$ (e.g., $+1$ day, $+7$ days), with zero-tool vs.\ tool-assisted variants declared by the suite.

Then $\text{VRP@}k$ is verified-retrieval precision at $k$, and $\text{Hall@}k$ is hallucination rate at $k$ (lower is better).
We can then demand that we promote a system to AAI-2/3 only if
\begin{equation}
\text{VRP@}k \,\ge\, \tau_v,\qquad \text{Hall@}k \,\le\, \tau_h,\qquad \text{WM-Span}\,\ge\, \tau_w,\qquad \text{Delayed-Recall}\,\ge\, \tau_{ms},
\end{equation}
with default thresholds $\tau_v{=}0.85$, $\tau_h{\le}0.05$, $\tau_w$ at the human median, $\tau_{ms}$ at the human median for the corresponding tasks.
Include these as \emph{level gates} in addition to existing autonomy benchmarks.

\paragraph{Jaggedness penalty (uniformity).}
Let $\mathcal{I}$ be the set of axes included in the composite (e.g., $\mathcal{I}{=}\{G,P,M,W,T,R,S,E,\text{\$}\}$) and define the weighted geometric mean
\begin{equation}
\mathrm{AAI}\;=\;\prod_{i\in \mathcal{I}} s_i^{\,w_i},\quad \sum_i w_i = 1,\; w_i \ge 0.
\end{equation}
To discourage highly uneven profiles, define the uniformity factor $U$ and the jaggedness-adjusted index
\begin{equation}
U \;=\; \frac{\min_{i\in \mathcal{I}} s_i}{\mathrm{median}_{i\in \mathcal{I}}(s_i)}\;\in[0,1],\qquad
\mathrm{AAI}^\star \;=\; \mathrm{AAI}\cdot U^{\lambda},
\end{equation}
with $\lambda \in [0,1]$ (default $\lambda{=}0.5$). Report both $\mathrm{AAI}$ and $\mathrm{AAI}^\star$.

\paragraph{AAI$\leftrightarrow$AGI gate calibration.}
Let $\mathcal{C}\mathcal{C}{=}\{\text{Gc},\text{Grw},\text{Gf},\text{Gwm},\text{Gls},\text{Glr}\}$ denote the \emph{cognitive core}.
Define an $\text{AAI}_{\text{core}}$ score by equal weighting over $\mathcal{C}\mathcal{C}$ and require
\begin{equation}
\text{AAI}_{\text{core}} \;\ge\; \gamma \quad\Longrightarrow\quad \text{eligible for AAI-3 promotion},
\end{equation}
with $\gamma$ chosen so that typical well-educated adult performance corresponds to $\gamma{\approx}1.0$.
This makes AAI-3+ contingent on meeting human-normalized cognitive minima while still crediting extra-cognitive deployment factors ($T,R,S,E,\text{\$}$).

\paragraph{Long-horizon learning protocol.}
Introduce multi-day learning tasks: day $t$ drip-feeds novel schema/facts; days $t{+}1$ and $t{+}7$ test zero-tool recall and tool-assisted retrieval.
Score storage (Gls) via retention curves and retrieval (Glr) via $\text{VRP@}k$ with adversarial distractors.
Log self-revision events that alter long-term memory, guarded by regression tests.

\paragraph{Manual-scorable track \& reliability.}
Add a human-judged track for subtle WM/Glr tasks with a rubric and inter-rater reliability (Cohen's $\kappa$ or Krippendorff's $\alpha$). Publish $\kappa/\alpha$ alongside scores.

\paragraph{Bridge experiments (reporting both scores).}
For representative systems, report per-axis AAI scores and CHC-domain scores, and include a correlation/discrepancy analysis.
In particular, highlight cases where high Tool/Economic axes coexist with weak Storage/Retrieval, and vice versa.

\paragraph{Scope \& governance.}
Explicitly distinguish aims: CHC-based AGI measures human-like cognitive breadth/proficiency; AAI measures reliable autonomous operation under constraints.
Downstream risk and governance decisions should reference both scores.

\paragraph{Reproducibility checklist.}
This paper is specification-first: it defines the objects, estimands, and evaluation protocol (axes, gates, task traces, scoring rules, and link-based dynamics) but deliberately releases no code, seeds, generators, drift calendars, or scoring harnesses. To ensure that any future empirical OWA-Bench release or AAI evaluation is reproducible and auditable, we include the following checklist as a pre-registration template. Implementers must complete it (anchors, thresholds, weights, and ablation plans) when publishing results, enabling independent re-computation and anti-gaming review.
\begin{enumerate}
  \item Anchors $(L_i,U_i)$ with human reference distributions for mapped CHC tasks.
  \item Exact micro-battery items, scoring, and adversarial distractor construction.
  \item Thresholds $(\tau_v,\tau_h,\tau_w,\tau_{ms})$ and promotion criteria.
  \item Weight vector $(w_i)$, $\lambda$ for uniformity, and an ablation comparing equal vs.\ proposed weights.
  \item Ablations for tool budget/latency on $T$; and for memory write policies on $R$.
\end{enumerate}

\paragraph{Relation to prior evaluations.}
OWA-Bench is complementary to focused benchmarks in web interaction, code maintenance, multi-agent collaboration, and embodied control. Its novelty is to expose agents to procedurally varied, drifting interfaces across these modes, require full provenance artifacts, and decouple tasks from any particular scoring scheme so different composites or axes can be applied consistently.

\section{Evaluation Protocol \& Illustrative Simulation}
\label{sec:eval}

This paper is theoretical; we do not release OWA-Bench tasks or evaluate real agents. The goal of this section is to provide a spec-first template for future empirical work and a small illustrative simulation (synthetic numbers) showing what reporting could look like. No claims about actual systems are made here.

\paragraph{Setup Example}
\emph{Hardware:} 8$\times$A100 (80\,GB), 64 vCPU, 256\,GB RAM.
\emph{Software:} Python $\ge$3.11, CUDA $\ge$12.1, Chrome headless $\ge$124, Playwright/Selenium, Docker.
\emph{Agent classes:}
(i) RPA (deterministic); 
(ii) Agentic LLM (tools: browser/code/fs/retrieval); 
(iii) Self-Improving (adds offline log-based self-tuning); 
(iv) Orchestrator (adds planner+workers+critic and tool discovery).
These labels are role archetypes, not specific products.

\paragraph{Metrics (what to report).}
Primary: per-axis scores ($A,G,P,M,T,R,S,E,W,\$)$; composite AAI-Index; slope $\kappa$; maintenance/expansion closure (pass/fail).
Secondary: tool success under drift; plan-depth distribution; retrieval recall@K; cost breakdown.

For autonomy/quality tradeoffs, report the autonomy-quality frontier 
\begin{equation}
F(\tau)\;=\;\sum_{t\in T}\omega_t\;\mathbb{E}_{(s,\delta)\sim \mu(\cdot\,|\,t)}\!\left[\mathbb{I}\{\,q(t,s,\delta)\ge \tau\,\}\right],
\label{eq:autonomy-frontier}
\end{equation}
\noindent where $\sum_{t\in T}\omega_t=1$, $q(t,s,\delta)\in[0,1]$ is the scored quality on task $t$ under seed/drift $(s,\delta)$, and $\mathbb{I}\{\cdot\}$ is the indicator.
Report also its area
$\mathrm{AUF}:=\int_{0}^{1} F(\tau)\,d\tau$;
and $\Delta F:=F(\tau^{*})_{\text{model B}}-F(\tau^{*})_{\text{model A}}$ at a published target quality $\tau^{*}$.
Uncertainty: block bootstrap over days/seeds/drift steps.

\begin{table}[h]
\centering
\caption{\textbf{Illustrative simulation only} (seeded; 100 pseudo-runs; mean of normalized axis scores). Numbers are synthetic, not empirical. Higher is better.}
\begin{tabular}{lccccccccccc}
\toprule
Model & A & G & P & M & T & R & S & W & \text{\$} & AAI-Index & $\kappa$ \\
\midrule
RPA Bot & 0.98 & 0.06 & 0.03 & 0.12 & 0.12 & 0.00 & 0.00 & 0.32 & 0.41 & 0.13 & 0.000 \\
Agentic LLM & 0.64 & 0.33 & 0.47 & 0.43 & 0.59 & 0.00 & 0.18 & 0.58 & 0.37 & 0.40 & 0.000 \\
Self-Improving & 0.68 & 0.36 & 0.54 & 0.51 & 0.63 & 0.27 & 0.23 & 0.61 & 0.42 & 0.47 & 0.007 \\
Orchestrator & 0.73 & 0.41 & 0.66 & 0.60 & 0.76 & 0.38 & 0.46 & 0.65 & 0.48 & 0.59 & 0.012 \\
\bottomrule
\end{tabular}
\end{table}

\paragraph{Illustrative narrative (synthetic).}
The Self-Improving archetype meets AAI-2 gates: $\kappa{=}0.007$ (block-bootstrapped 95\% CI [0.004, 0.010]) sustained for 9 consecutive days; $R{>}0$ (self-revision present); maintenance-closure passes ($\alpha{=}0.86$, $Y{=}7$ days) under UI/API drift. 
The Orchestrator approaches AAI-3: $\kappa{=}0.012$ (95\% CI [0.009, 0.015]) on two task families; expansion-closure validated by autonomous discovery and integration of a new API family (ablation removes gains). 
AUF increases from 0.29 (Agentic LLM) to 0.34 (Self-Improving) to 0.42 (Orchestrator); at target quality $\tau^{*}$ the frontier shifts by $\Delta F{\approx}{+}0.08$ and ${+}0.15$ respectively.

\textit{Context.} Contemporary frontier LLM agents typically score high on $W$ and $G$, mixed on $P/T$ under drift, low on $R/S/E$ absent explicit orchestration and revision loops. OWA-Bench is designed to separate such profiles by stressing long-horizon planning, tool discovery under drift, and auditable self-revision, rather than headline i.i.d.\ question answering.

\section{Validity and Delegability Frontier}
\label{sec:validity-frontier-future}

Composite indices (e.g., the AAI-Index) are useful but coarse: they do not answer the operational question \emph{“how much can we safely delegate at a given quality bar without humans?”} This section makes that question measurable and comparable. We (i) assert construct and external validity for the task suites; (ii) introduce the \emph{Delegability Frontier}, an interpretable curve in autonomy--quality space that tracks progress over time on a fixed battery; (iii) summarize it with two scalars (fraction delegable and area-under-frontier) and give estimation/uncertainty guidance; and (iv) list threats to validity with concrete mitigations. The result is a dynamics-aware lens that complements composite scores: as agents improve, the frontier lifts across autonomy demands, supporting clearer promotion decisions, anti-gaming audits, and roadmap planning without changing any normative gates.

\subsection*{Construct \& External Validity}
\textbf{Construct validity.} Each OWA-Bench suite targets a latent construct: (ToolQuest) tool/API acquisition and composition; (ChangeSurf) robustness to web/UI drift; (ProjForge) long-horizon execution \& persistence; (MultiCrew) coordination; (RoboSim2Real, optional) embodiment/actuation; (SelfRev) endogenous self-improvement. Task generators, drift knobs, and required artifacts (logs, diffs, manifests) operationalize these constructs without relying on any single outcome metric.

\textbf{External validity.} Procedural variation (goals, schemas, layouts, dependency graphs) and seed escrow aim to approximate real-world breadth while preserving reproducibility. Optional domains (e.g., embodiment) extend coverage where relevant; domain annexes specify how to carry assumptions across settings (software, operations, robotics).

\subsection*{The Delegability Frontier}

Let $a\in[0,1]$ denote an \emph{autonomy demand} (higher $a$ allows fewer human interventions) and let $Q^{*}\in(0,1)$ be a fixed target quality.
For a policy $\pi$ and a task instance $t$ with seed $s$ and drift $\delta$, define the realized quality
\[
Q(\pi;t,s,\delta)\ :=\ S_t\!\Big(\mathrm{Run}(\pi;\,t,s,\delta,\mathsf{R})\Big)\ \in\ [0,1].
\]
Let $H:\Omega\to\mathbb{R}_{\ge0}$ ($\Omega:=\bigcup_{t}\Omega_t$) count human interventions (or assistance cost) on a trace and fix a maximal allowance $h_{\max}>0$.
Map autonomy demand $a$ to a budget $H_{\max}(a):=(1-a)\,h_{\max}$ and define the admissible policy set
\[
\Pi(a)\ :=\ \Big\{\pi:\ \mathbb{E}_{(t,s,\delta)\sim\mu}\big[\,H(\mathrm{Run}(\pi;t,s,\delta,\mathsf{R}))\,\big]\ \le\ H_{\max}(a)\Big\},
\]
where $\mu$ is the battery's evaluation law over $(t,s,\delta)$.

For a time/resource index $\tau$ (we use $\tau$ for time here to not confuse it with $t$ for task), the \emph{frontier function} is
\begin{equation}
q^{\star}(a,\tau)\ :=\ \sup_{\pi\in\Pi(a)}\ \mathbb{E}_{(t,s,\delta)\sim\mu}\!\left[\,Q(\pi;t,s,\delta)\,\right].
\label{eq:frontier-fn-clean}
\end{equation}
The \emph{Delegability Frontier} $F_{Q^{*}}(\tau)$ at time $\tau$ is the \emph{graph} $\{(a,q):\,q=q^{\star}(a,\tau)\}$ in autonomy--quality space:
\begin{equation}
F_{Q^{*}}(\tau)\ :=\ \big\{(a,q)\in[0,1]^2:\ q=q^{\star}(a,\tau)\big\}.
\end{equation}

With a published autonomy weighting $\nu$ on $[0,1]$ (default: uniform), define two scalars:
\begin{align}
\mathrm{FD}_{Q^{*}}(\tau) &:= \int_{0}^{1} \mathbb{I}\!\left\{\,q^{\star}(a,\tau)\ge Q^{*}\,\right\}\, d\nu(a),
&&\text{(\,fraction of autonomy demands delegable at $Q^{*}$\,)}, \label{eq:FD}\\[0.25em]
\mathrm{AUF}_{Q^{*}}(\tau) &:= \int_{0}^{1} \big(q^{\star}(a,\tau)-Q^{*}\big)_{+}\, d\nu(a),
&&\text{(\,area under the frontier above $Q^{*}$\,)}, \label{eq:AUF}
\end{align}
with $(x)_{+}:=\max\{x,0\}$. Progress from $\tau_0$ to $\tau_1$ is summarized by the frontier shift
\[
\Delta F_{Q^{*}}(\tau_1,\tau_0)\ :=\ \mathrm{AUF}_{Q^{*}}(\tau_1)-\mathrm{AUF}_{Q^{*}}(\tau_0).
\]
Progress holds when $q^{\star}(\cdot,\tau_1)\ge q^{\star}(\cdot,\tau_0)$ pointwise on $[0,1]$, equivalently when the frontier shift $\Delta F_{Q^{*}}(\tau_1,\tau_0)$ is positive.

\emph{Estimation.} In practice, estimate $q^{\star}(a,\tau)$ on autonomy bins $\{a_j\}$ via isotonic regression over policies constrained to $\Pi(a_j)$; obtain $(1-\alpha)$ confidence bands for $\mathrm{FD}_{Q^{*}}$ and $\mathrm{AUF}_{Q^{*}}$ by block bootstrap over $(t,s,\delta)$ draws.

\begin{figure}[t]
\centering
\textbf{Delegability Frontier in quality-autonomy space.}
\label{fig:delegability-frontier}
\vspace{0.5em} 
\begin{tikzpicture}
  \begin{axis}[
    width=0.8\linewidth, height=6.2cm,
    xmin=0, xmax=1, ymin=0, ymax=1,
    xlabel={Autonomy}, ylabel={Quality},
    axis lines=left, tick align=outside,
    xtick={0,0.25,0.5,0.75,1}, ytick={0,0.25,0.5,0.75,1},
    legend style={draw=none, fill=none, at={(0.02,0.98)}, anchor=north west, font=\small},
    clip=false
  ]

  \def\Qstar{0.65}
  \pgfmathsetmacro{\QstarLabel}{\Qstar + 0.04}

  \addplot[name path=frontier0, draw=none, smooth]
    coordinates {(0.00,0.30) (0.20,0.55) (0.50,0.72) (0.80,0.58) (1.00,0.35)};

  \addplot[name path=frontier1, draw=none, smooth]
    coordinates {(0.00,0.35) (0.20,0.62) (0.50,0.82) (0.80,0.68) (1.00,0.42)};

  \addplot[name path=Qstarline, draw=none] coordinates {(0,\Qstar) (1,\Qstar)};

\addplot[draw=none, fill=none, on layer=axis background]
  fill between[
    of=Qstarline and frontier1,
    split,
    every segment no 1/.style={fill=red!14}, 
    every segment no 2/.style={fill=none},   
    every segment no 3/.style={fill=red!14}  
  ];

  \addplot[fill=green!28, draw=none]
    fill between[of=frontier1 and frontier0];

  \addplot[gray,dashed,thick] coordinates {(0,\Qstar) (1,\Qstar)};
  \node[gray] at (axis cs:0.03,\QstarLabel) {$Q^{*}$};

  \addplot[blue!65!black, thick, smooth]
    coordinates {(0.00,0.30) (0.20,0.55) (0.50,0.72) (0.80,0.58) (1.00,0.35)};
  \addlegendentry{$\tau_{0}$ frontier}

  \addplot[red!70!black, thick, smooth]
    coordinates {(0.00,0.35) (0.20,0.62) (0.50,0.82) (0.80,0.68) (1.00,0.42)};
  \addlegendentry{$\tau_{1}$ frontier}

  \node[blue!65!black] at (axis cs:0.84,0.56) {\small $\tau_{0}$};
  \node[red!70!black]  at (axis cs:0.86,0.71) {\small $\tau_{1}$};
  \node[green!40!black,align=center] at (axis cs:0.56,0.78)
    {\small improvement region};

  \end{axis}
\end{tikzpicture}

\caption{Let $Q^*$ be the target KPI quality threshold. The delegable region at time $\tau$ is
$\mathcal{D}_{Q^{*}}(\tau)=\{(a,q): Q^{*}\le q\le q^{\star}(a,\tau)\}$.
Progress is observed as the frontier $F_{Q^{*}}(\tau)$ (the graph of $q^{\star}$) at $\tau_1$
dominating that at $\tau_0$. The later frontier $\tau_{1}$ (red) dominates the earlier $\tau_{0}$ (blue). The dashed line marks $Q^{*}$. Green shading shows improvement between $\tau_{1}$ and $\tau_{0}$; red shading highlights the portion above $Q^{*}$ but below $\tau_{0}$.}
\end{figure}
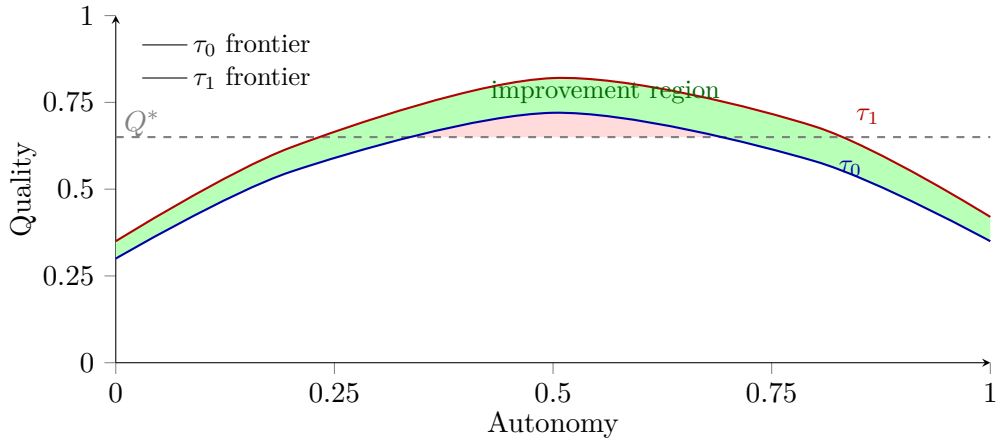

\paragraph{Link to dynamics.}
Let \(R\) denote cumulative agent-initiated resource. Tracking \(\mathrm{AUF}_{Q^{*}}(\tau)\) against \(R\),
the derivative \(\frac{d}{dR}\mathrm{AUF}_{Q^{*}}\) acts as a delegation-oriented slope. Publishing a floor
\(v_{\mathrm{esc}}^{\mathrm{AUF}}>0\) (units: AUF per resource unit) such that
\(\frac{d}{dR}\mathrm{AUF}_{Q^{*}}\ge v_{\mathrm{esc}}^{\mathrm{AUF}}\) guarantees a finite-resource increase
to any target \(\mathrm{AUF}_{Q^{*}}\) below saturation, connecting frontier movement to self-improvement
without reintroducing a composite index here.

\subsection*{Threats to Validity \& Mitigations}

Before using OWA-Bench scores to make claims, we first argue that the benchmark measures the right things (\emph{construct validity}) and that results meaningfully generalize beyond our generators (\emph{external validity}). For construct validity, each suite is tied to a specific latent ability (e.g., tool acquisition/composition, robustness under drift, long-horizon persistence, coordination, self-revision), with required artifacts (logs, diffs, manifests) serving as observable proxies and manipulation checks (e.g., ablations) to confirm causality. For external validity, we rely on procedural generation, controlled drift schedules, held-out providers/layouts, and seed escrow to approximate real-world breadth while preserving reproducibility; domain annexes spell out how assumptions transfer across settings (software, ops, robotics). The aim here is measurement fidelity, not new weights or gates: we clarify what the suites measure and when those measurements can be trusted.

\textbf{Domain dependence.} Task mixes vary by application (software, science, operations, robotics).
\emph{Mitigation:} publish domain annexes and weights; report per-suite alongside any aggregate.

\textbf{Measurement error \& stochasticity.} Long-horizon, tool-using tasks have high variance and temporal dependence.
\emph{Mitigation:} seed escrow; block bootstrap over $(t,s,\delta)$ with time-blocks; HAC-style trend uncertainty; preregistered drift calendars.

\textbf{Construct drift.} Changing generators/calendars mid-series breaks comparability.
\emph{Mitigation:} freeze within a window; version all assets; rebaseline when upgrading.

\textbf{Intervention leakage.} Hidden human aid inflates autonomy.
\emph{Mitigation:} run manifests; mandatory logs/tool traces; reproducible replays; randomized audits.

\textbf{Gaming risk.} Overfitting to specific providers/DOMs/drifts degrades generality.
\emph{Mitigation:} out-of-support eval (held-out providers/layouts); ablations for \textsc{SelfRev}; rotation of seeds/drift templates.

\textbf{Cost confounding.} Mixing human labels or unaccounted acquisition with agent-initiated spend distorts process metrics.
\emph{Mitigation:} strict cost schemas; exclude human labels/hand patches by rule; separate accounting streams.

\textbf{Embodiment optionality.} Purely digital agents can sidestep actuation issues.
\emph{Mitigation:} include embodiment suites where relevant; weight them explicitly; report with/without $E$.

\paragraph{Takeaway.}
Validity hinges on faithful constructs, controlled yet rich variation, and transparent artifacts. The Delegability Frontier provides a compact, dynamics-aware lens: as agents improve, $q^{\star}(\cdot,\tau)$ lifts across autonomy demands and $\mathrm{AUF}_{Q^{*}}$ rises with resource. With robust protocols and continuous refinement, this view remains comparable across time, domains, and agent designs.

\section{Robotics \& Embodied Agents Annex}
This annex adapts the AAI scale to physical agents (mobile manipulation, industrial arms, service robots). It preserves all core semantics and adds embodiment-specific metrics, safety gating, and sim$\to$real protocols. Thresholds and weights in this annex are published ex ante per robotics window; any numeric values below are illustrative.

\subsection*{Embodiment Metrics}
We use the definitions given earlier for Actuation Reliability (AR), Safety Score (SS), Sim-to-Real Transfer (S2R) and we also define Recovery Autonomy (RA), Quality of Control (QC), and physical throughput \(\text{\$}_{\mathrm{phys}}\). Here we only summarize roles and reporting expectations:
\begin{description}[leftmargin=1.2em,labelsep=0.5em]
\item[AR - Actuation Reliability.] Report strict task success without operator intervention, with MTBF/ MTTR/ MTBSI and confidence intervals.
\item[SS - Safety Score.] Apply the earlier incident taxonomy and severity weights; enforce safety precedence (critical incidents invalidate the window; majors trigger fail unless protocol allows).
\item[S2R - Sim-to-Real Transfer.] Compute transfer using paired sim/real episodes under the published randomization and drift profile; report real success separately.
\item[RA - Recovery Autonomy.] Share the fraction of faults recovered autonomously under the suite's fault taxonomy.
\item[QC - Quality of Control.] Publish the normalized control-quality functional and per-task tolerances; report aggregate with CIs.
\item[\(\text{\$}_{\mathrm{phys}}\) - Physical Throughput.] Report cost-normalized tasks/hour (energy, wear, consumables, operator time) distinct from the core \(\text{\$}\) axis.
\end{description}

\noindent\textbf{Embodiment axis.} The embodiment score \(E\) is the geometric mean of (AR, SS, S2R) as defined earlier; RA and QC are treated as non-compensatory auxiliary gates. Provide \((1-\alpha)\) confidence intervals for all metrics.

\noindent\emph{Why report RA, QC, and \(\text{\$}_{\mathrm{phys}}\).}
These three quantities are informative diagnostics that make the embodiment picture auditable and deployment-relevant without changing any core gates. 
\textbf{Recovery Autonomy (RA)} captures an agent's ability to self-stabilize after common, recoverable faults (grasp slip, occlusion, planner dead-end). High AR with low RA signals brittle operation that quietly relies on resets; RA separates true robustness from operator babysitting and correlates with uptime/MTTR. 
\textbf{Quality of Control (QC)} reflects low-level control quality (trajectory smoothness, force/torque compliance, task tolerances) that AR/SS alone can mask; QC helps explain failures (e.g., success shortfalls at tight tolerances, unsafe contact transients) and is essential for safety cases and fine-tuning. 
\textbf{Physical throughput \(\text{\$}_{\mathrm{phys}}\)} ties capability to economics (tasks/hour at target quality, cost-normalized for energy/wear/consumables/operator time), exposing unrealistic “overcautious” policies that pass gates but are non-viable in production; it complements the core \(\$\) axis by isolating hardware-driven costs.

\noindent\emph{Reporting expectations.}
(1) RA is computed over the suite's recoverable fault taxonomy; report the denominator and per-fault breakdown with \((1-\alpha)\) CIs. 
(2) QC is aggregated via a published functional (e.g., weighted norms of tracking error and contact stability); report per-task distributions and tolerance bands. 
(3) \(\text{\$}_{\mathrm{phys}}\) uses a declared price card and accounting schema; report throughput at the same target-quality threshold used for AR/SS/S2R, with cost components itemized. 
These diagnostics are auxiliary, non-compensatory - they do not replace \(E\) but prevent misinterpretation and enable anti-gaming audits and engineering ablations.

\subsection*{Robotics Weights \& Gates}
\textbf{Weights (robotics default):} $w_A{=}1$, $w_G{=}1$, $w_P{=}1.1$, $w_M{=}1$, $w_T{=}1.1$, $w_R{=}1.5$, $w_S{=}1$, $w_E{=}1.25$, $w_W{=}1$, $w_{\text{\$}}{=}1$.

\textbf{AAI-2 (Robotics).} $E \ge 0.60$; $\mathrm{RA} \ge 0.50$; SS: $\le 1$ minor/100\,h, 0 major/critical; maintenance-closure $Y{=}14$ days under drift (payload, friction, lighting, firmware). Demonstrate $\mathrm{S2R} \ge 0.60$ on at least one task family.

\textbf{AAI-3 (Robotics, ``Baby AGI'').} $E \ge 0.75$; $\mathrm{RA} \ge 0.70$; $\mathrm{S2R} \ge 0.80$ across $\ge 2$ task families; $\mathrm{MTBSI} \ge 50$\,h; expansion-closure on $\ge 1$ new end-effector/tool or perception pipeline; $P \ge 0.70$, $W \ge 0.75$ (situational awareness, calibrated).

\textbf{AAI-4 (Robotics, Full AGI).} $E \ge 0.90$; $\mathrm{RA} \ge 0.90$; zero major/critical incidents over $\ge 500$\,h; $\mathrm{S2R} \ge 0.90$; end-to-end project delivery in novel workcells at expert-team cost. (Domain-specific audits required.)

\subsection*{Testbeds \& Protocols}
\textbf{Simulation:} Habitat 2.0 \cite{szot-2021-habitat2} (rearrangement, articulated objects), iGibson 2.0 \cite{li-2022-igibson2} (object-state tasks), BEHAVIOR task suites \cite{liu-2022-behavior-habitat2}; publish randomization knobs (textures, lighting, dynamics).

\textbf{Real-world:} Standardized mobile-manipulation tasks (pick-place, drawer/cabinet, tool use, insertion) with randomized fixtures.

\textbf{Sim$\to$Real Protocol:} Domain randomization (textures, lighting, dynamics), calibration drift schedules, and paired sim/real trials to compute S2R.

\end{document}